\documentclass{article}

\PassOptionsToPackage{numbers, compress}{natbib}

\usepackage[main, final]{neurips_2026}

\usepackage{amsmath,amsfonts,bm}
\usepackage[most]{tcolorbox}
\usepackage{siunitx}
\usepackage{graphicx}
\usepackage{subcaption}

\usepackage[utf8]{inputenc} 
\usepackage[T1]{fontenc}    
\usepackage{hyperref}       
\usepackage{url}            
\usepackage{booktabs}       
\usepackage{nicefrac}       
\usepackage{microtype}      
\usepackage{xcolor}         

\usepackage{floatflt} 
\usepackage{paralist} 
\usepackage{float}
\usepackage{amsthm}
\usepackage{amssymb}
\usepackage{listings}  
\usepackage{xcolor}
\usepackage{hyperref}
\usepackage{indentfirst}
\usepackage{booktabs}
\usepackage{caption}
\usepackage{calligra}
\usepackage{fancyhdr}
\usepackage{pifont}
\usepackage{colortbl} 
\usepackage{MnSymbol}%
\usepackage{amsbsy}
\usepackage{multirow}
\usepackage{wrapfig}
\usepackage{bm}
\usepackage{afterpage}
\usepackage{mathtools}
\usepackage{threeparttable}
\usepackage{setspace}
\usepackage{algorithm}

\usepackage{nicematrix, tikz} 

\newtheorem{prop}{Proposition}

\newlength{\boxunit}
\newtcolorbox{lightgrey}{
  colback=gray!10, 
  colframe=gray!60,
  boxrule=0.2pt,      
  auto outer arc,     
  width=\linewidth,
  left=0.7mm,
  right=0.7mm,
  top=0.3\boxunit,
  bottom=0.3\boxunit,
  before skip=0.3\boxunit,
  after skip=0.3\boxunit
}

\newtcolorbox{greybox}{
  boxsep=0mm, 
  top=1.5\boxunit,
  bottom=1.5\boxunit,
  before skip=\boxunit,
  after skip=\boxunit,
  width=\linewidth,
  left=2.5mm,
  right=2.5mm,
}

\def\eqref#1{equation~\ref{#1}}

\def\1{\bm{1}}

\DeclareMathAlphabet{\mathsfit}{\encodingdefault}{\sfdefault}{m}{sl}
\SetMathAlphabet{\mathsfit}{bold}{\encodingdefault}{\sfdefault}{bx}{n}

\def\gB{{\mathcal{B}}}

\def\gI{{\mathcal{I}}}

\def\gL{{\mathcal{L}}}

\def\gR{{\mathcal{R}}}

\def\gV{{\mathcal{V}}}

\def\gX{{\mathcal{X}}}

\newcommand{\softmax}{\mathrm{softmax}}

\definecolor{Gray}{gray}{0.9}
\definecolor{LightGray}{gray}{0.97}

\usepackage[utf8]{inputenc} 
\usepackage[T1]{fontenc}    
\usepackage{hyperref}       
\usepackage{url}            
\usepackage{booktabs}       
\usepackage{amsfonts}       
\usepackage{nicefrac}       
\usepackage{microtype}      
\usepackage{xcolor}         

\usepackage{algorithm}
\usepackage{algpseudocode}
\usepackage{subcaption} 

\usepackage[capitalize]{cleveref}

\crefname{equation}{eq.}{eqs.}
\Crefname{equation}{Eq.}{Eqs.}
\crefname{section}{sec.}{secs.}
\Crefname{section}{Sec.}{Secs.}

\theoremstyle{plain}
\newtheorem{theorem}{Theorem}[section]

\theoremstyle{definition}
\newtheorem{definition}[theorem]{Definition}

\theoremstyle{remark}

\title{SemanticDLM+: Improving Diffusion LLMs through Bias-variance Trade-off in Transition Kernel Design}

\author{%
  \bfseries
  Keyue Jiang$^{1,4,*, \dagger}$,\quad
  Yuxiang Wang$^{1,2,*, \dagger}$,\quad
  Xiang Yu$^{1,2,*, \dagger}$,\quad
  Yanan Zhao$^{3}$,\quad
  Qifang Zhao$^{1}$ \\
  \bfseries Bohan Tang$^{5}$, \quad Baojian Zhou$^{2}$,\quad
  Yanghua Xiao$^{2}$,\quad
  Lin Qu$^{1}$,\quad
  Xiaoxiao Xu$^{1}$ \\[1.5ex]
  $^{1}$Alibaba Group, $^{2}$Fudan University, $^{3}$Nanyang Technological University,\\
   $^{4}$ University College London, \quad $^{5}$ University of Oxford\\
  $^*$ Equal Contribution. $^\dagger$ Work done during an internship at Alibaba Group.
}

\begin{document}

\maketitle

\begin{abstract}
Diffusion Language Models (DLMs) have demonstrated strong scaling capacity as alternatives to autoregressive language models. However, their performance is highly sensitive to the choice of transition kernels, and poorly designed kernels can lead to issues like training instability, slow convergence, and biased sampling. In this paper, we study this sensitivity through a principled analysis of generalization error and identify three critical factors: asymptotic bias (difficulty in approximating the posterior distribution), exposure bias (error propagation during sampling), and optimization variance induced by kernel dispersion. We further compare different transition kernels: masking diffusion yields sparse and easier posterior-approximation targets, while uniform diffusion provides stronger sampling-side repair but induces harder approximation. Motivated by this trade-off, we revisit a previously overlooked variant, semantic DLM (SemDLM), where the transition kernel corrupts tokens to neighborhoods that are semantically similar. Our theory suggests that SemDLM can serve as a plausible middle ground by reducing the posterior approximation difficulty of uniform diffusion while retaining repair ability. However, we find that SemDLM suffers from a semantic basin problem, where sampling repeatedly stays within a semantic region and produces low-diversity text. To address this, we propose SemDLM+, which adds a global transition and a semantic-frequency penalty during sampling. Experiments on LM1B and OpenWebText show that SemDLM+ improves training dynamics and achieves competitive language modeling and generation quality with satisfactory diversity.
\end{abstract}

\section{Introduction}

Diffusion language models (DLMs) \citep{dream2025, nie2025llada} have emerged as a compelling alternative to autoregressive language models (ALMs)~\cite{Dubey2024TheL3,Yang2024Qwen25TR,DeepSeekAI2024DeepSeekV3TR} due to their parallelizable training and faster decoding speeds. DLMs need to design a transition kernel to gradually corrupt the clean data to noise, and different kernels can lead to substantially different training and sampling dynamics. The predominant archetype is absorbing kernel~\citep{sahoo2024simple, DBLP:conf/nips/ShiHWDT24, ou2025your}, as prior work shows that masking DLMs~\citep{austin2021structured,hoogeboom2021argmax} can effectively alleviate slow convergence, training instability, and weak generalization~\citep{wang202610openchallengessteering} compared to others such as uniform, marginal, and semantic neighborhood diffusion~\citep{austin2021structured,hoogeboom2021argmax,DBLP:conf/iclr/SchiffSPWBDARPK25,qin2025defog}. However, recent studies demonstrated that uniform DLMs can enjoy better scaling capacity when sufficient data and training budget are given~\citep{vonruette2025scalingbehaviordiscretediffusion, DBLP:conf/icml/RutteFDOS025, sahoo2026scalingmaskeddiffusionlanguage}. This reveals a gap that is not fully explained: 
\begin{center}
\textit{RQ1: Why does uniform diffusion scale well under large resources, while masking diffusion remains stronger in many practical regimes? Can a transition kernel combine the advantages of both?}
\end{center}

In this paper, we answer this question by developing a principled error analysis framework for DLMs. We first decompose the generation error into \textit{Approximation Error}, \textit{Sampling Error}, and \textit{Forward Kernel Mismatch}, and then separate the first two through a bias-variance lens. This highlights three kernel-dependent factors: asymptotic bias that reflects approximation difficulty, exposure bias that measures error accumulation during reverse sampling, and optimization variance that captures finite-resource instability. Our analysis uncovers a trade-off in existing paradigms: masking diffusion has sparse posterior targets that are easier to fit, but provides limited intrinsic repair during sampling. 
Uniform diffusion has denser posterior targets that are harder to optimize, but can be preferable for sampling because its reverse dynamics naturally preserve the ability to repair earlier errors.

This trade-off motivates us to revisit a previously overlooked variant, namely semantic diffusion language models (SemDLMs), where the forward kernel corrupts a token into semantically related tokens. SemDLMs are theoretically attractive because they restrict the posterior to a meaningful neighborhood, reducing the approximation difficulty relative to uniform diffusion, while still allowing richer transitions than masking diffusion. However, prior SemDLM designs have not consistently delivered strong generation performance~\citep{austin2021structured,DBLP:journals/corr/abs-2603-21342}. This leads to our second question:
\begin{center}
\textit{RQ2: While theoretically plausible, why does SemDLM still underperform existing methods? And how can we translate the theoretical advantages of SemDLM into practical gains?}
\end{center}
In our experiments, we find a \emph{semantic basin} issue: the reverse sampling may repeatedly generate semantically adjacent tokens, producing locally plausible but low-diversity text. This happens because the semantic likelihood term and the model's rollout-induced bias can reinforce sampling within a same semantic cluster. To solve this issue, we propose SemDLM+. First, we add a global transition on top of the semantic transition kernel to prevent the sampling being trapped in some semantic neighborhoods. Second, we introduce a semantic-frequency penalty mechanism during sampling that counteracts the rollout-induced tendency to overproduce tokens from the same semantic basin. Together, these two mechanisms turn SemDLM into a powerful variants as DLMs: easier to train than fully uniform diffusion, but more repairable during sampling than purely masking-based diffusion.

In summary, our contributions are threefold. 1) We provide a principled error analysis for DLMs that explains how transition-kernel design affects approximation difficulty, sampling dynamics, and finite-resource optimization. 2) Guided by this analysis, we develop an improved SemDLM+, which augments SemDLM with a global transition for sampling repair and semantic-frequency penalty to avoid semantic-basin collapse. These designs make SemDLM practically successful and preserve its properties of efficient training and reliable sampling. 3) Experiments on LM1B and OpenWebText show that SemDLM+ improves training dynamics and achieves strong language modeling and generation performance, highlighting SemDLM+ as a promising direction for DLM kernel design.

\section{Preliminaries}
\label{sec:preliminaries}

We denote by $q_0$ the data distribution over support $\mathcal{X}$, and by $q_1$ a reference distribution that is easy to sample (e.g., the absorbing or uniform distribution). In language modeling, $\gX$ represents the space of length-$L$ sequences where $x=\left(x^{(1)}, \ldots, x^{(L)}\right) \in \mathcal{X}:=\mathcal{V}^L$ with $\gV$ being a vocabulary of size $|\gV|=V$. Diffusion models aim to construct a probability path $q_t, 0\leq t\leq 1$ such that one can sample from \(q_1\) and transform it through the learned reverse process to get samples that approximately follow \(q_0\).

\textbf{Diffusion as Continuous-time Markov Chains (CTMC).} Following~\citet{DBLP:conf/icml/LouME24}, we view the forward noising process to construct the probability path as a CTMC with \textit{infinitesimal generator} $Q_t$, i.e., \(\frac{d q_t}{d t}=Q_t q_t,0\leq t\leq 1\). One can simulate the forward CTMC via:
\begin{equation}
\text{Forward Process via Euler Sampling: }q\left(x_{t+dt}=y \mid x_t=z\right)=\delta_{z y}+Q_t(z, y) dt+O\left(dt^2\right)
\end{equation}

Diffusion deploys a parameterized model to mimic the reverse process, \(p_\theta(x_{t-dt} \mid x_t) \approx q(x_{t-dt} \mid x_t)\), such that one can iteratively sample a trajectory from the reference distribution to the data distribution through $p_\theta(x_{t-dt}\mid x_t)$. A predominantly used parameterization in DLMs is $x$-prediction~\citep{nie2025llada, dream2025, cheng2025sdarsynergisticdiffusionautoregressionparadigm, liu2025wedlmreconcilingdiffusionlanguage}, which builds $p_\theta\left(x_0 \mid x_t\right)$ instead of directly approximating \(q(x_{t-dt} \mid x_t)\) as:
\begin{equation}
\label{eq:x-prediction}
p_\theta\left(x_{t-dt} \mid x_t\right):=\int_{x_0 \in \mathcal{X}} q\left(x_{t-dt} \mid x_t, x_0\right) p_\theta\left(x_0 \mid x_t\right)dx_0.
\end{equation}
The posterior is \(q\left(x_{t-dt} \mid x_t, x_0\right) \propto q\left(x_t \mid x_{t-dt}\right) q(x_{t-dt} \mid x_0)\). We note that the local transition \(q_{t\mid t-dt}\), the cumulative forward \(q_{t\mid 0}\), and the generator $Q_t$ are equivalent representations of the same forward process. As such, we describe the process in terms of \(q_{t\mid 0}\) and $x$-prediction in following.

\textbf{Training objective.} DLMs optimize over a variational upper bound of negative log-likelihood~\citep{ho2020ddpm}, $-\log p_\theta(x_0) \le \ell_0 +\ell_{\text{prior}} + \sum_{t} \ell_t$, with $\ell_t =  \mathbb{E}_{x_t}[D_{KL}(q(x_{t-dt}|x_t, x_0) \| p_\theta(x_{t-dt}|x_t))]$, $\ell_0 =\mathbb{E}_{q(x_{0:1}|x_0)} [-\log p_\theta(x_{0:1})]$ and $\ell_{\text{prior}} = D_{KL}(q(x_1|x_0) \| p_\theta(x_1))$. Taking $dt \to 0$ will make the first two terms negligible. With~\Cref{eq:x-prediction}, we can derive \(D_\mathrm{KL}\left(q\left(x_{t-dt} \mid x_t\right) \| p_\theta\left(x_{t-dt} \mid x_t\right)\right)  \leq D_\mathrm{KL}\left(q\left(x_0 \mid x_t\right) \| p_\theta\left(x_0 \mid x_t\right)\right)\)~\citep{li2023on}, 
which then yields the common training objective for DLMs:
\begin{equation}
\label{eq:ori_loss}
\gL(\theta)=\mathbb{E}_{t, x_0, x_t \sim q(x_t\mid x_0)} D_\mathrm{KL}\left(q\left(x_0\mid x_t\right) \| p_\theta\left(x_0 \mid x_t\right)\right).
\end{equation}
\textbf{Sampling Objective.} After learning $p_\theta(x_0\mid x_t)$, the synthesized data points $x_0$ are generated by iteratively sampling from the induced reverse kernel \(x_{t-dt}\sim p_\theta(x_{t-dt}\mid x_t)\) from $t=1$ to 0. The quality of synthesized samples is measured by the generation risk:
\begin{equation}
\label{eq:generation_error}
    \gR(\theta) = D_\mathrm{KL}\left(q\left(x_0\right) \| p_\theta\left(x_0\right)\right).
\end{equation}
\textbf{Transition Kernel Design.} With the CTMC framework, we can interpret DLM variants via their transition kernel design. In this paper, we mainly consider the following variants of \(q(x_t \mid x_0)\).

\textbf{Absorbing (Masking) Transition} defines a special token $\mathtt{[MASK]}$ as the absorbing state such that,
\begin{equation}
    q(x_t = j \mid x_0 = i) = \alpha_t  \delta_{ij} + (1 - \alpha_t)  \delta_{j, \mathtt{[MASK]}},
    \label{eq:masking_kernel}
\end{equation}
where $\alpha_t$ is the decay factor and $\delta_{ij}$ is the Kronecker delta.

\textbf{Uniform transition.}
Uniform diffusion spreads the corrupted mass over the full vocabulary:
\begin{equation}
q(x_t=j\mid x_0=i)
=
\alpha_t\delta_{ij}
+ V^{-1}(1-\alpha_t),
\qquad i,j\in\mathcal V .
\label{eq:uniform_kernel}
\end{equation}
\textbf{Semantic Transition} is first introduced in~\citet{austin2021structured}, which has its kernel defined as,
\begin{equation}
    q(x_t = j \mid x_0 = i)
    =
    \alpha_t \delta_{ij}+
    (1-\alpha_t) s_t^{\mathrm{sem}}(j\mid i),
    \label{eq:knn_kernel}
\end{equation}
where the semantic kernel was initially designed as transitions over semantic clusters, such that $s_t^{\mathrm{sem}}(j\mid i)=(k_t)^{-1}\mathbb I\bigl(j\in \mathcal N_{k_t}(i)\bigr)$ where $\mathcal N_{k_t}(i)$ is the top-$k_t$ semantic neighborhood of token $i$. Unfortunately, this design suffers from strong training-sampling mismatch~\citep{DBLP:conf/iclr/NingLSSE24} that leads to significant performance degeneration. So we can lightly modify the kernel to make sure the reference distribution matches in forward and reverse process. This gives the transition as
\begin{equation}
    s_t^{\mathrm{sem}}(j\mid i)
    =\exp(\tau_t^{-1}\mathrm{sim}(i,j)) / \sum_{k\in\mathcal V}\exp(\tau_t^{-1}\mathrm{sim}(i,k)),
    \label{eq:semantic-proposal}
\end{equation}
where $\mathrm{sim}(i,j)$ is a similarity score between embeddings for word $i$ and $j$; $\tau_t>0$ is a temperature parameter scheduled to monotonically increase from $\tau_t \rightarrow 0 ,t \rightarrow 0$ and $\tau_t \rightarrow +\infty, t \rightarrow 1$.

\section{Principled Error Analysis in DLMs through Bias-Variance Trade-off}
\label{sec:error_analysis}

In this section, we first build a principled error analysis for DLMs (\Cref{sec:error_analysis_framework}), where we identify several sources that cause the generation error. Then, we provide an interpretation of the transition kernel designs' impact on algorithm behaviors in~\cref{sec:bias-variance-from-loss}.

\subsection{Error Analysis and Bias-variance Trade-off in DLMs}
\label{sec:error_analysis_framework}

We can decompose the generation error in~\cref{eq:generation_error} as (detailed proof in~\cref{apdx:bias_decomp}):
\begin{equation}
\label{eq:kl-three-term-avg}
\begin{aligned}
&D_{\mathrm{KL}}\big(q(x_0)\|p_\theta(x_0)\big)
=\mathbb{E}_{t}\Big[D_{\mathrm{KL}}\big(q(x_t)\|p_\theta(x_t)\big)\Big]+\\
&\mathbb{E}_{t, x_t\sim q_t}\left[
D_{\mathrm{KL}}\big(q(x_0\mid x_t)\|p_\theta(x_0\mid x_t)\big)
\right]-\mathbb{E}_{t,x_0\sim q}\left[
D_{\mathrm{KL}} \big(q(x_t\mid x_0)\| p_\theta(x_t\mid x_0)\big)
\right].
\end{aligned}
\end{equation}
Effectively, 1) \textbf{Approximation Error.} The second term is the approximation error under the true marginals
$\{q_t\}$, which is our training objective in~\Cref{eq:ori_loss}. 2) \textbf{Sampling Error.} The first term is a time-averaged marginal mismatch between the model's roll-out $\{p_t\}$ and the true marginals $\{q_t\}$. 3) \textbf{The forward path mismatch.} The last term is a correction term measuring the mismatch between the model's forward conditionals $p(x_t\mid x_0)$ and the true forward process $q(x_t\mid x_0)$. 
\begin{lightgrey}
\begin{prop}[Error Decomposition]
It is common to assume that $p(x_t\mid x_0)\equiv q(x_t\mid x_0)$ for all $(x_0,t)$ as the forward paths are manually designed. Then the generation error becomes,
\begin{equation}
\label{eq:kl-two-term-avg}
D_{\mathrm{KL}}\big(q(x_0)\|p(x_0)\big)= \underbrace{\mathbb{E}_{t}\Big[D_{\mathrm{KL}}\big(q_t\|p_t\big)\Big]}_{\textbf{Sampling Error}}
+\underbrace{\mathbb{E}_{t, x_t}\left[
D_{\mathrm{KL}}\big(q(x_0\mid x_t)  \|  p(x_0\mid x_t)\big)
\right]}_{\textbf{Approximation Error}}.
\end{equation}
\end{prop}
\end{lightgrey}

\paragraph{Bias-Variance Decomposition.} In general ML, the generalization error, a measurement that gives the prediction ability of an ML algorithm, can be decomposed into 3 meaningful terms\footnote{Irreducible risks are usually introduced by the noise, so we omit them in the following analysis. }:
\[
\text{Generalization Error}=\text {Bias}+\text{Variance}+ \text{Irreducible Risk}
\]
\emph{Bias} is the error between the model's expected prediction and the ground truth, primarily stemming from limitations in the algorithm design or hypothesis space. \emph{Variance} quantifies the spread of the estimated model parameters around their expected value and its impact on inference, typically arising from sensitivity to finite data sampling and the stochasticity of the optimization process. 

We can utilize a similar framework to respectively decompose the approximation and sampling error from~\Cref{eq:kl-two-term-avg} (details in Appendix~\ref{apdx:bvd}). 
\begin{lightgrey}
\begin{prop}[Bias-Variance Trade-off]
Let \(\bar p(x_0\mid x_t)
:= \mathbb E_S\big[p_{\hat \theta_S}(x_0\mid x_t)\big] =
\mathbb E_S\big[\hat p_S(x_0\mid x_t)\big]\) denote the expected predictive distribution,
where $\hat \theta(S)$ is the parameter induced by a specific training set and optimization randomness $S$ (short as $\hat p_S := p_{\hat \theta_S}$). The approximation error in~\Cref{eq:kl-two-term-avg} can be decomposed as,
\begin{equation}
\label{eq:kl-bv-global}
\mathbb{E}_{t,x_t}\mathbb E_S\Big[ D_{\mathrm{KL}}(q(x_0\mid x_t) \| \hat p_S(x_0\mid x_t))\Big]
=
\underbrace{
\mathbb E_{t, x_t}\Big[
D_{\mathrm{KL}}\big(q(x_0\mid x_t)\| \bar p(x_0\mid x_t)\big)
\Big]}_{\textbf{Asymptotic Bias } \gB_{\text{asym}}}
+
\underbrace{
\mathbb E_{t, x_t}\big[\mathcal V(x_t)\big]}_{\textbf{Variance } \gV},
\end{equation}
where $x_t$ is constructed over the forward process, $q(x_t)=\int_{x_0}q(x_t\mid x_0)p(x_0)dx_0$. The posterior approximation variance term is \(\gV(x_t)
:=
\mathbb E_{q(x_0\mid x_t)}
\Big[
\log \bar p(x_0\mid x_t) -
\mathbb E_S \log (\hat p_S(x_0\mid x_t))
\Big]
\ge 0\). And the step-wise sampling error in~\Cref{eq:kl-two-term-avg} can be written as,
\begin{equation}
\label{eq: sampling_decomp}
D_{\mathrm{KL}}\big(q_t(x_t)\|p_t(x_t)\big) = \underbrace{D_{\mathrm{KL}}\big(q_t(x_t)\|\bar{p}_t(x_t)\big)}_{\textbf{Exposure Bias }\gB_{t}} + \underbrace{\mathbb{E}_{x_t}[\log \bar{p}(x_t) - \mathbb{E}_S (\log \hat{p}_S(x_t))]}_{\textbf{Sampling Roll-out Variance }\gV_t}
\end{equation}
\end{prop}
\end{lightgrey}
\Cref{apdx:unify_ts_variance} shows that sampling variance in~\Cref{eq: sampling_decomp} can be controlled by approximation variance. We therefore unify posterior and roll-out variances in the following analysis.

\subsection{How Transition Kernel Design affects Generation Error?}
\label{sec:bias-variance-from-loss}


From~\Cref{sec:error_analysis_framework}, we effectively identified three important sources of error: asymptotic bias in training, exposure bias in sampling, and the variance. In this section, we will illustrate how different transition designs introduced in~\Cref{sec:preliminaries} affect the model performance of DLMs through these terms. We first provide an overview of the three sources.

\begin{lightgrey}
\textbf{Asymptotic Bias.} $\gB_{\text{asym}}$ reflects the intrinsic difficulty of approximating the true denoising posterior \(q(x_0\mid x_t)\) with a parameterized predictor \(p_\theta\) even with infinite training data and optimal convergence. It comes from \textit{Model Misspecification} and \textit{Capacity Bottlenecks}, especially when the architecture lacks the expressivity to represent the complexity of the true noise-corrupted posterior.

\textbf{Exposure Bias.} $\mathcal{B}_{\exp}
:= \sum_{t=0}^{T} \gB_t$. While sampling, the model generates the sequence iteratively from $t=1$ to $t=0$. The input $x_t$ is not drawn from the true marginal $q_t$, but from the model's own previous generative distribution $\bar{p}_t$. If the distribution $\bar{p}_t$ deviates even slightly from $q_t$, then subsequent reverse steps are evaluated on shifted inputs. This phenomenon is known as Exposure Bias, and it quantifies the error coming from \textit{Sampling Dynamics} and \textit{Error Accumulation}.

\textbf{Variance.} $\gV$ measures the instability of the learned predictor under finite training resources, which captures the error caused by \textit{Finite Data}, \textit{Optimization Stochasticity}, and \textit{Resource-limited Training}. Different transition kernels \(q(x_t\mid x_0)\) can induce different levels of variance.
\end{lightgrey}

\subsubsection{Asymptotic Bias Analysis through Posterior Geometry.}
\label{sec:asy_bias}

To describe the posterior approximation difficulty, we can characterize the \emph{local geometry} of the target posterior. Let \(f_\theta(x_t)^i \in \mathbb R^{|V|}\) denote the predicted logits at position \(i\in[1:L]\) such that $p_\theta^i(\cdot\mid x_t)=\softmax \bigl(f_\theta(x_t)^i\bigr)$.
Assuming conditional factorization across positions given \(x_t\), the asymptotic bias can be decomposed into local losses:
\begin{equation}
\gB_{\text{asym}}
=
\mathbb E_{t,x_t} \left[
\sum_{i=1}^L
\ell \left(f_\theta(x_t)^i; q(x_0^i\mid x_t)\right)
\right],
\quad
\ell(f;q):=D_{\mathrm{KL}} \left(q(x_0^i\mid x_t)\,\|\,\softmax(f)\right),
\label{eq:global-kl-local-decomp}
\end{equation}
Applying the chain rule, the parameter gradient can be written as,
\begin{equation}
\nabla_\theta \gL
=
\mathbb E_{t,x_t}
\left[
\sum_{i=1}^n
\sum_{k\in V}
\left(
p_\theta(x_0^i=k\mid x_t)-q(x_0^i=k\mid x_t)
\right)
\cdot
\nabla_\theta f_\theta(x_t)^i_k
\right].
\label{eq:param-grad-local-mismatch}
\end{equation}

To formalize the local approximation difficulty, let \(f^\star(q)\) denote optimal logit vector satisfying $\softmax \bigl(f^\star(q)\bigr)=q$. We apply a local Taylor expansion which yields
\begin{equation}
\nabla_f \ell(f;q)
=
\Sigma(q)(f- f^*(q))
+o \left(\|f - f^*(q)\|\right), \text{with }\Sigma(q):=\nabla^2_f\ell(f;q)=\text{Diag}(q)-qq^\top.
\label{eq:local-grad-linearization}
\end{equation}
where $\Sigma(q)$ is the softmax Hessian. Thus, the posterior approximation difficulty is fully determined by the following properties of \(\Sigma(q)\).

\begin{lightgrey}
\paragraph{Posterior approximation difficulty.}
For a target posterior \(q:=q(x_0\mid x_t)\), we quantify its local approximation difficulty via the following metrics of \(\Sigma(q)\):
\begin{equation}
\begin{aligned}
\text{Logits Active Direction:  }&d_{\mathrm{act}}(q):=
\text{rank} \bigl(\Sigma(q)\bigr)
=
|\text{supp}(q)|-1,
\\
\text{Logits Error Sensitivity: } &\gI_1(q) :=
\text{tr} \bigl(\Sigma(q)\bigr)
=1-\|q\|_2^2, \quad\gI_2(q):= \text{tr} \bigl(\Sigma(q)^2\bigr).
\end{aligned}
\end{equation}
\(d_{\mathrm{act}}\) measures the number of independent logit directions that must be simultaneously fitted. \(\mathcal I_1\) quantifies how local logit mismatch is converted into error, and \(\mathcal I_2\) quantifies gradient energy. Lower values mean easier approximation. We provide more explanations in~\cref{apdx:asym_bias}. 
\end{lightgrey}

Based on these metrics, we can now study how the shape of the posterior
\(q(x_0\mid x_t)\) affects the difficulty of fitting. By Bayes' rule, \(q(x_0\mid x_t)
\propto
p_{\mathrm{data}}(x_0)\, q(x_t\mid x_0)\). Therefore, the geometry of the posterior is determined jointly by \(p_{\mathrm{data}}(x_0)\) and the forward kernel \(q(x_t\mid x_0)\). Substituting the kernel for masking (\Cref{eq:masking_kernel}), uniform (\Cref{eq:uniform_kernel}), and semantic diffusion (\Cref{eq:knn_kernel}) obtains:

\begin{lightgrey}
\begin{prop}[Approximation difficulty across diffusion variants.]
\label{prop:appx_diff}
We derive in~\Cref{apdx:asym_bias} that, 
\begin{equation}
\label{eq:strict-gap-dact}
\mathbb E \left[d_{\mathrm{act}}^{\mathrm{mask}}\right]
\leq
\mathbb E \left[d_{\mathrm{act}}^{\mathrm{sem}}\right]
\mathbb E 
\leq\left[d_{\mathrm{act}}^{\mathrm{uni}}\right], \quad
\mathbb E \left[\gI_1^{\mathrm{mask}}\right]
\leq
\mathbb E \left[\gI_1^{\mathrm{sem}}\right]
\leq
\mathbb E \left[\gI_1^{\mathrm{uni}}\right]
\end{equation}
\end{prop}
where superscripts suggest diffusion variants. This suggests two complementary regimes. When compute or model capacity is limited, uniform diffusion is harder to optimize because it forces the model to fit dense posteriors at nearly every position (large $d_{\text{act}}$), and any local mismatch would lead to large generalization error (large $\gI_1$). Masking diffusion and semantic diffusion are easier to optimize as they have a more restricted posterior space and activate fewer logit directions. In the data-sufficient and compute-abundant regime, however, these additional active directions can become useful supervision:  masking diffusion receives little learning signal from visible tokens, whereas uniform diffusion keeps almost every position informative for training. This partially explains why uniform diffusion scales better but underperforms other variants in most scenarios.
\end{lightgrey}

\subsubsection{The Impact of Generators on Exposure Bias Propagation}

\paragraph{Exposure Bias.} \(\mathcal{B}_{\exp}
:=
\sum_{t=0}^{T} \gB_t\) measures the mismatch accumulated along the generation trajectory. Unlike the asymptotic bias in the approximation stage, exposure bias propagates through the reverse generator itself. We give the following proposition and leave the derivation in~\Cref{apdx:exposure_bias_proof}. 

\begin{lightgrey}
\begin{prop}[Exposure Bias Propagation]
\label{prop:exposure_bias}
With mild conditions, the propagation satisfies:
\begin{equation}
\gB_{t-dt}
\le
\eta_t \,\gB_t + \rho_t,
\label{eq:main-exp-bias-recursion}
\end{equation}
where \(\rho_t\) upper-bounds the step-wise error:
\( \rho_t \geq \sup_{x_t}
D_{\mathrm{KL}}\bigl(q_t(x_{t-dt}\mid x_t) \| p_t^\theta(x_{t-dt}\mid x_t)\bigr) \) and \(\eta_t^{\mathrm{KL}}\in[0,1]\) is the error propagation coefficient for forwarding \(q_t\) to \(q_{t-dt}\). In particular,
\begin{equation}
\eta_{t,\mathrm{mask}} \approx 1,
\quad
\eta_{t,\mathrm{uni}}
\le
1-\lambda_t^{\mathrm{uni}}<1,
\quad
\eta_{t,\mathrm{sem}}
\le
1-\lambda_t^{\mathrm{sem}}<1. \nonumber
\end{equation}
Consequently, for \textbf{semantic and uniform diffusion}, \(\rho_t\le \rho\) and \(\eta_t \le \eta<1\), then \(\mathcal{B}_{\exp} = \sum_{t}\gB_t = \mathcal{O}\!\left(\frac{T\,\rho}{1-\eta}\right)\)
; whereas for \textbf{masking diffusion} \(\eta_t \approx 1\) and \(\mathcal{B}_{\exp} = \mathcal{O}\!\left(T^2 \rho\right)\) with sampling step \(T\).
\end{prop}
\end{lightgrey}
\textbf{Implication.} The coefficient \(\eta_t\) measures how strongly sampling errors propagate across reverse steps. For masking diffusion, \(\eta_t\approx 1\), so early errors will accumulate quadratically w.r.t \(T\) and can become ``early commitments'' that later steps cannot easily repair. Remasking can mitigate this but cannot completely fix the issue. By contrast, uniform diffusion and semantic diffusion have \(\eta_t<1\), so the reverse dynamics can contract accumulated error and provide an intrinsic repair mechanism, which leads to linear error accumulation w.r.t \(T\).

\subsubsection{The Impact of Generators on Variance}
\label{sec:gradient-noise-heterogeneity}

Then we explain how the transition kernel affects optimization variance. 
Let \(\gL_{x_t}(\theta):=-\log p_\theta(x_0\mid x_t)\) denote the per-sample denoising likelihood and \(\gL_t(\theta):=\mathbb E_{x_t}[\gL_{x_t}(\theta)]\) the per-step loss. Let \(\theta^\star\) minimize \(\gL(\theta)=\mathbb E_t[\gL_t(\theta)]\) so that \(\mathbb E_t[\nabla_\theta \gL_t(\theta^\star)]=0\). We look into the gradient variance as a surrogate for the total variance (Details in Appendix~\ref{apdx:gradient_variance}). At \(\theta^\star\), the law of total variance yields
\begin{equation}
\begin{aligned}
\mathrm{Var}_{t,x_t}\big(\nabla_\theta \gL_{x_t}(\theta^\star)\big)
=
\underbrace{
\mathbb E_t\left[\mathrm{Var}_{x_t}\big(\nabla_\theta \gL_{x_t}(\theta^\star)\big)\right]
}_{\text{within-step noise}}
+
\underbrace{
\mathrm{Var}_t\left(\nabla_\theta \gL_t(\theta^\star)\right)
}_{\text{between-step heterogeneity}}.
\label{eq:total-variance-grad}
\end{aligned}
\end{equation}
The first term is the usual gradient noise within a fixed step $t$. The second term measures how the optimization tasks vary across diffusion times. This term is shaped by the transition kernel because \(\nabla_\theta \gL_t(\theta^\star)\) depends on \(q_t(x_t\mid x_0)\) through the corrupted input \(x_t\) and the induced posterior target \(q(x_0\mid x_t)\). Therefore, kernels with more dispersed or more time-varying corruption patterns can induce larger between-t heterogeneity (such as a uniform). We prove this connection in Appendix~\ref{apdx:gradient_variance}, and \cref{fig:variance_vs_dispersion} empirically validates the correlation between gradient variance and the transition kernel dispersion across time \(t\). This suggests that a dispersed kernel will induce training instability.
\begin{lightgrey}
\textbf{Take-home Message for~\Cref{sec:error_analysis}.}
The analysis in this section reveals a fundamental trade-off in the transition kernel design for DLMs. \textit{Masking diffusion} is easy to train because the absorbing kernel induces sparse denoising posteriors, but it lacks intrinsic repair and can suffer from sampling error propagation. \textit{Uniform diffusion} repairs such errors through global transition, but its dense and dispersed posteriors make training harder and noisier. \textit{Semantic diffusion} is therefore a natural middle ground: semantic locality can reduce approximation difficulty, while global transition can preserve repair ability.
\end{lightgrey}

\section{Semantic Diffusion: Towards Bias-Variance Minimization in DLMs}
\label{sec:semantic_diffusion}

\subsection{Revisiting Semantic Diffusion}
\label{sec:desired_properties}

The analysis in Sec.~\ref{sec:error_analysis} suggests that semantic diffusion has the potential to mitigate the approximation difficulty of uniform diffusion, and can avoid error accumulation in sampling. However, prior studies report that semantic diffusion often underperforms alternatives such as uniform or marginal diffusion~\citep{austin2021structured}. Although a concurrent work~\citep{DBLP:journals/corr/abs-2603-21342} reports extremely low test perplexity with a semantic kernel, it does not evaluate unconditional generation. In our experiments, when testing the generation ability of semantic diffusion, we consistently observe a semantic basin problem.
\begin{lightgrey}
    \begin{definition}[Semantic Basin]
        We define a semantic basin as a local semantic region in which the reverse chain repeatedly samples semantically adjacent tokens, producing locally plausible but low-diversity text. An example is shown in~\cref{table:semantic Basin}.
    \end{definition} 
\end{lightgrey}

\begin{table}[htbp]
    \centering
    \caption{An example of text generation that suffers from semantic basin.}
    \label{tab:generated_text}
    \newcolumntype{L}[1]{>{\raggedright\arraybackslash}p{#1}}
    
    \begin{tabular}{@{} l L{0.75\textwidth} @{}}
        \toprule
        \textbf{Model} & \textbf{Generated Text} \\
        \midrule
        Semantic basin & 
         [CLS] Because bone was being made up, bone bone bone completely destroyed and dies in it. dies in it. dna … dies in it … the bone cell was not so “ dying ” “ the bone cell ” dies the bone marrow decays. \\
        Ours & 
        [CLS] the main aim is to increase the number of people who use the internet and reduce the amount of competition.  \\
        \bottomrule
    \end{tabular}
    \label{table:semantic Basin}
\end{table}

\paragraph{Why does semantic basin occur?}
Consider a token position \(i\), and let \(\hat x_t\) denote the current state during sampling. Ideally, the denoiser should approximate the clean posterior \(q(x_0^i=k\mid \hat x_t)\). By Bayes' rule, the posterior logit can be decomposed as
\begin{equation}
l_i^\star(k;\hat x_t)
=
\underbrace{
\log p_{\mathrm{data}}(x_0^i=k\mid \hat x_t^{-i})
}_{\text{contextual prior}}
+
\underbrace{
\log q_t(\hat x_t^i\mid x_0^i=k,\hat x_t^{-i})
}_{\text{local forward likelihood}}
+
\mathrm{const}.
\label{eq:ideal_logit_decomp}
\end{equation}
During sampling, however, \(\hat x_t\) is generated by the model itself rather than drawn from the true forward marginal. Therefore, the denoiser is evaluated on rollout states that may already contain accumulated errors. We write the actual sampling logit as
\begin{equation}
l_{\theta,i}(k;\hat x_t)
=
l_i^\star(k;\hat x_t)
+
\Delta_{\mathrm{roll},i}(k;\hat x_t)
+
\epsilon_i(k)
\label{eq:sampling_logit_decomp}
\end{equation}
where \(\Delta_{\mathrm{roll},i}\) denotes the logit bias induced by the current sampling trajectory. Let \(n_i^{(W)}(l)\) be the count of token \(l\) in a recent window \(W\), and let \(A(k,l)\) denote the contextual contribution of token \(l\) to the logit of token \(k\). In~\cref{apdx:semantic_basin_analysis}, we show that the rollout bias can be locally approximated as
\(\Delta_{\mathrm{roll},i}(k;\hat x_t)
\approx
\sum_{l\in V} A(k,l)n_i^{(W)}(l)\).
If \(A(k,l)>0\) for semantically related tokens \(k,l\in C\), then over-producing tokens from cluster \(C\) increases the logits of other tokens in the same cluster:
\[
n_i^{(W)}(l\in C)\uparrow
\Rightarrow
\Delta_{\mathrm{roll},i}(k)\uparrow \text{ for } k\in C
\Rightarrow
p_\theta(x_0^i\in C\mid \hat x_t)\uparrow.
\] 
Meanwhile, in a semantic diffusion kernel, the local likelihood \(q_t(j\mid k)\) is large when \(k\) is semantically close to \(j\). Thus, the likelihood term in~\Cref{eq:ideal_logit_decomp} also favors tokens within the same semantic cluster. These two effects reinforce each other and create a positive feedback loop, which drives the reverse chain into a semantic basin. We provide a formal analysis in~\cref{apdx:semantic_basin_analysis}.

\subsection{From SemDLM to SemDLM+: Mitigating Semantic Basins.}

The above analysis suggests that semantic basins can be mitigated by weakening this positive feedback. We use two complementary mechanisms to construct a negative feedback:

\textbf{(a) Global jumping in the forward kernel.}
Instead of using a purely local semantic transition, we add a global jumping component:
\begin{equation}
q_t(j\mid k,c)
=
\alpha_t\delta_{kj}
+
\textcolor{blue}{\beta_t\nu_t(j)}
+
(1-\alpha_t-\beta_t)s_t^{\mathrm{sem}}(j\mid k,c).
\label{eq:global_jump_kernel}
\end{equation}
The global transition \(\beta_t\nu_t(j)\) provides a mixing channel, preventing the model from absorbing into the local semantic cluster. The semantic kernel follows the KNN version as in~\Cref{eq:knn_kernel}, with $s_t^{\mathrm{sem}}(j\mid i)=(k_t)^{-1}\mathbb I\bigl(j\in \mathcal N_{k_t}(i)\bigr)$ where the number of semantic neighborhoods $k_t$ increases over time. Specifically, we set $\alpha_t=1-t$, $\beta_t=t^2$ and $k_t=1+\left(k_{\max }-1\right) t^\gamma$ with a hyperparameter $\gamma$.

\textbf{(b) Semantic-frequency Penalty.}
We further counteract rollout-induced positive feedback during sampling. Ideally, one would subtract the positive feedback term
\(\sum_{l\in V} [A(k,l)]_+ n_i^{(W)}(l)\).
However, \(A(k,l)\) is model- and context-dependent and is generally unavailable during sampling. We therefore use semantic similarity scores \(S_+(k,l)\ge 0\) as a practical surrogate, and define \(m_i^{(W)}(k)=\sum_{l\in V} S_+(k,l)n_i^{(W)}(l)\) as a semantic-frequency regularizer in sampling.

Taken together with a frequency penalty, we then apply a correction on the logits:
\begin{equation}
\tilde l_{\theta,i}(k)
=
l_{\theta,i}(k)
-
\textcolor{blue}{\lambda_{\mathrm{freq}}
\log\left(1+n_i^{(W)}(k)\right)}
-
\textcolor{blue}{\lambda_{\mathrm{sem}}
\log\left(1+m_i^{(W)}(k)\right)}.
\label{eq:hybrid_penalty_logit_main}
\end{equation}
The first penalty suppresses exact token over-production, while the second suppresses over-production of the semantic neighborhood of \(k\). This correction introduces a negative feedback mechanism that counteracts the rollout-induced positive feedback which leads semantic basins. Finally, we show in~\cref{apdx:omit_proofs} that SemDLM+ preserves the desirable properties of semantic diffusion: it reduces posterior approximation difficulty while retaining sampling-side repair ability.

\section{Experiments}

We now evaluate SemDLM+'s language modeling and text generation. We first outline the setup in~\cref{sec: exp_setup} and compare model-size/data-budget scaling and generation quality in~\cref{sec:main_results}. We then examine component effects and semantic feedback in~\cref{sec:aba_study}.

\subsection{Experiment settings}
\label{sec: exp_setup}


\paragraph{Training Setup.} Following~\citet{arriola2025block}, we train two variants of SemDLM+ on the One Billion Words dataset (LM1B~\citep{chelba2014billion}) and OpenWebText (OWT~\citep{Gokaslan2019OpenWeb}) at various model scales. Models trained on both datasets use the \emph{bert-base-uncased} tokenizer. For both datasets, we set up the maximum training steps to be 200K global steps. We fixed the context length to be 128 for LM1B and 1024 for OWT. For the results that we reproduced, we utilize a change-aware loss, which we found helpful in stabilizing training. We refer to Appendix~\ref{sec:knn_approximation} for a detailed model setup and algorithm implementation details. 

\textbf{Generation Setup.} We follow semi-autoregressive sampler. Each stride is initialized from the global transition prior, and each denoising step applies Continuous-time Markov Chain
sampling (CTMC) sampler with tau-leaping and predictor-corrector~\cite{DBLP:conf/nips/CampbellBBRDD22}. We vary the sampling steps ranging $\{128, 256, 512\}$ for LM1B and $\{128, 256, 512, 1024\}$ for OWT, and set the number of strides as 2. 

\textbf{Evaluation Setup.} Given the issue of semantic basin, perplexity (PPL) and entropy together serve as the evaluation metrics for language modelling and generation ability. We draw the generation frontier and calculate the area-under-the-curve (AUC) as the metric. Testing PPL is measured on texts produced by the model-native sampler, and Generation PPL is measured by GPT-2-Large.
x`

\subsection{Main Results}
\label{sec:main_results}

\begin{wraptable}{r}{0.35\textwidth}
  \centering
  \small
  \setlength{\tabcolsep}{4pt}
  \label{tab:exp2-fits}
  \begin{tabular}{lrrr}
    \toprule
    Method & $L_\infty$ & $\alpha_N$ & $\alpha_D$ \\
    \midrule
    MDLM    & 2.9478 & 0.7528 & 0.5337 \\
    Uniform & 3.9891 & 0.4525 & 0.7220 \\
    \rowcolor{Gray}
    SemDLM+ & 4.0223 & 0.4201 & 0.8541 \\
    \bottomrule
  \end{tabular}
    \caption{Scaling-law coefficients.}
\end{wraptable}

\paragraph{Scaling Law via Likelihood Evaluation.} We compare the scaling laws of MDLM, Uniform DLMs, and SemDLM+ on OWT at four scales, $N \in\{0.1, 0.3, 0.6, 1.0\}$ B and $D/N\in\{2,5,10,20\}$, where $N$ counts non-embedding parameters and $D$ is training tokens. We utilize the Kaplan form, where we fit the parameter and data coefficients \(\alpha_N\) and \(\alpha_D\) towards NLL, with \(\gL(N,D)=\gL_\infty+
[(\frac{N_c}{N})^{\alpha_N/\alpha_D}
+\frac{D_c}{D}]^{\alpha_D}\). ~\Cref{tab:exp2-fits} reports the fitted coefficient over three training seeds per combination. Prop.~\ref{prop:appx_diff} motivates the expected orderings parameter bound of $\alpha_N^{\mathrm{MDLM}} > \alpha_N^{\mathrm{SemDLM}} > \alpha_N^{\mathrm{Uniform}}$ and the data bound of $\alpha_D^{\mathrm{MDLM}} > \alpha_D^{\mathrm{SemDLM}} > \alpha_D^{\mathrm{Uniform}}$. The fitted parameter exponents instead give 
$\alpha_N^{\mathrm{MDLM}} > \alpha_N^{\mathrm{Uniform}}
\approx \alpha_N^{\mathrm{SemDLM}}$, while the fitted data exponents give 
$\alpha_D^{\mathrm{SemDLM}} \geq \alpha_D^{\mathrm{Uniform}} > \alpha_D^{\mathrm{MDLM}}$. This suggests a better scaling law for SemDLM.

\begin{wrapfigure}{r}{0.35\textwidth}
  \centering
  \includegraphics[
    width=\linewidth,
  ]{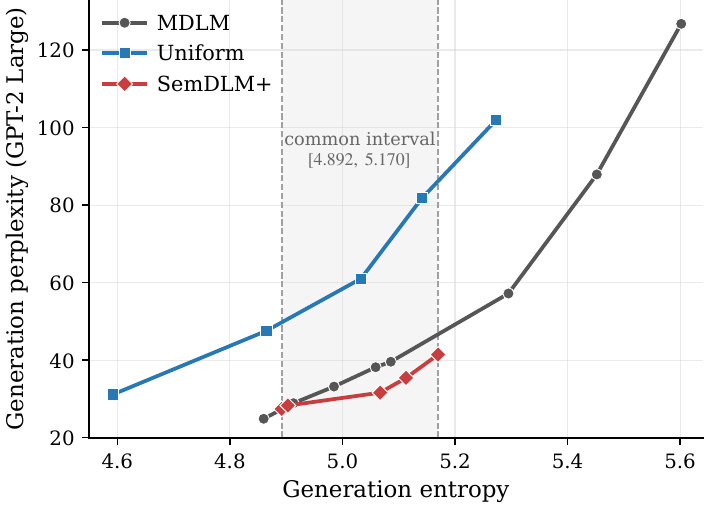}
  \caption{Gen-PPL - entropy frontier curves.}
  \label{fig:ppl_entropy_frontier}
  \vspace{-30pt}
\end{wrapfigure}


\paragraph{Generation Quality Evaluation.}
Since semantic basin can produce deceptively low PPL with low diversity, we compare generation frontiers and AUC in~\cref{fig:ppl_entropy_frontier}. We record a common range  of entropy $[4.892,5.170]$ by sweeping sampling temperature. The AUC of the Gen-PPL curve is 8.95 for SemDLM+, 10.14 for MDLM, and 17.95 for Uniform (lower is better). This supports a favorable overall trade-off for SemDLM+, although MDLM is slightly better near the lower boundary.

\subsection{Behavior Analysis}
\label{sec:aba_study}

\textbf{The Impacts of Generator Dispersion}. In our theoretical results in~\Cref{sec:asy_bias} and \ref{sec:gradient-noise-heterogeneity}, we show that reducing the dispersion of generator $q_t$ can effectively reduce approximation difficulty and stabilize training. In SemDLM+, this is reflected in the choice of $\gamma$. Large $\gamma$ induces a more complex optimization landscape and lower dispersion. Thus, we systematically analyze the impact of $\gamma$ in~\cref{fig:test_ppl} and~\ref{fig:gen_ppl} and validate this hypothesis. With a larger $\gamma$, the training convergence significantly improves. However, with $\gamma$ overly enlarged, the generation quality degrades, as the transition is constrained in the semantic basin without any escape, thus becoming harmful for the sampling.

\paragraph{Component-wise Effects.}
\Cref{tab:component_ablation} reports a full $2^3$ ablation on 0.1B OWT models, where we carefully evaluate the effects of global jumping (G), semantic frequency penalty (S) and exact frequency penalty (F). All F-off settings collapse or yield at most one sample with $H\ge4.8$ out of 75. With F enabled, adding G reduces Gen-PPL from $104.73\pm14.81$ to $28.34\pm1.39$, consistent with a non-local repair channel. Adding S to G+F reduces semantic repetition from $0.142\pm0.006$ to $0.123\pm0.006$, while increasing Gen-PPL to $34.48\pm1.09$. Thus the semantic penalty trades PPL for less semantic repetition.

\begin{table}[t]
  \centering
  \caption{Full $2^3$ ablation on 0.1B OWT models. G: global jumping; S/F: semantic-/exact-frequency penalties. Collapse denotes severe low diversity; $\pm$ gives reported run means and sample SD.}
  \label{tab:component_ablation}
  \small
  \setlength{\tabcolsep}{4pt}
  \begin{tabular}{ccccccc}
    \toprule
    G & S & F & Gen-PPL $\downarrow$ & $H\ge4.8$ & Mean entropy & Semantic repetition $\downarrow$ \\
    \midrule
    off & off & off & Collapse & 1/75 & 1.96 & -- \\
    off & on  & off & Collapse & 1/75 & 2.73 & -- \\
    off & off & on  & $104.73\pm14.81$ & 52/75 & 4.79 & -- \\
    off & on  & on  & $129.88\pm20.37$ & 54/75 & 4.86 & -- \\
    on  & off & off & Collapse & 0/75 & 4.13 & $0.348\pm0.009$ \\
    on  & on  & off & Collapse & 1/75 & 4.18 & $0.227\pm0.005$ \\
    on  & off & on  & $28.34\pm1.39$ & 57/75 & 4.90 & $0.142\pm0.006$ \\
    \rowcolor{Gray} on & on & on & $34.48\pm1.09$ & 58/75 & 4.93 & $0.123\pm0.006$ \\
    \bottomrule
  \end{tabular}
  \par\smallskip
\end{table}

\paragraph{The Impacts of Semantic Frequency Penalty.}

\cref{fig:gen_ppl_lambda} and~\ref{fig:gen_entropy_lambda} further illustrate the role of the semantic frequency penalties. We can observe an extremely low generation ppl and entropy when \(\lambda_{\mathrm{freq}}=0\), which coincides with the results in~\cite{DBLP:journals/corr/abs-2603-21342} but this actually indicates a collapsed low-diversity regime rather than genuinely strong generation. Increasing \(\lambda_{\mathrm{freq}}\) raises entropy and prevents exact-token repetition, while \(\lambda_{\mathrm{sem}}\) further suppresses repetition within a semantic neighborhood. However, overly strong penalties can over-flatten the clean posterior and worsen PPL. This shows that these penalties trade off likelihood and diversity by counteracting the semantic basin.


\begin{figure*}[t]
  \centering

  \begin{minipage}[t]{0.35\textwidth}
    \vspace{0pt}
    \centering
    \includegraphics[width=\linewidth]{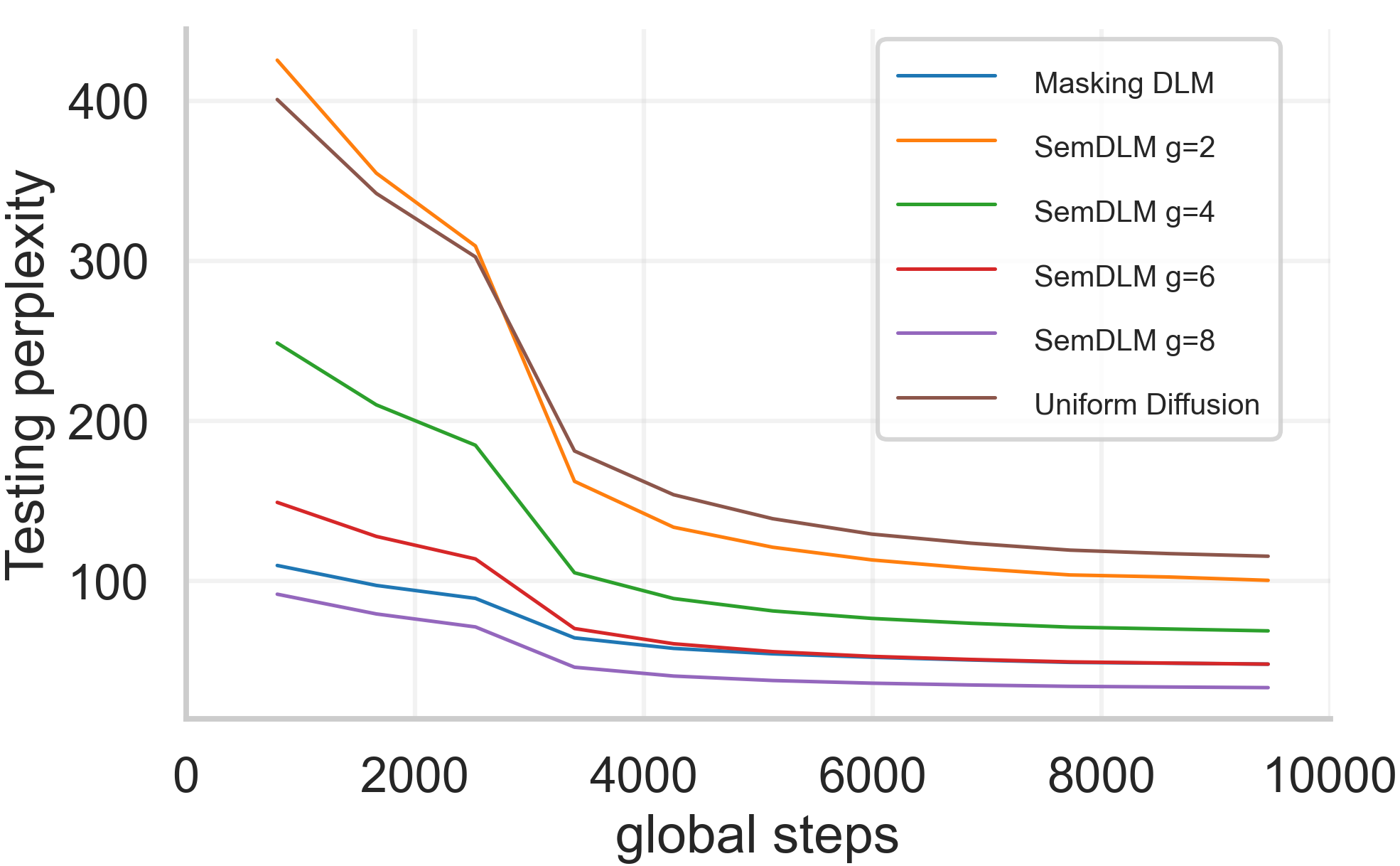}
    \captionof{figure}{Testing PPL with varying $\gamma$ during training in LM1B.}
    \label{fig:test_ppl}
  \end{minipage}
  \hfill
  \begin{minipage}[t]{0.35\textwidth}
    \vspace{0pt}
    \centering
    \includegraphics[width=\linewidth]{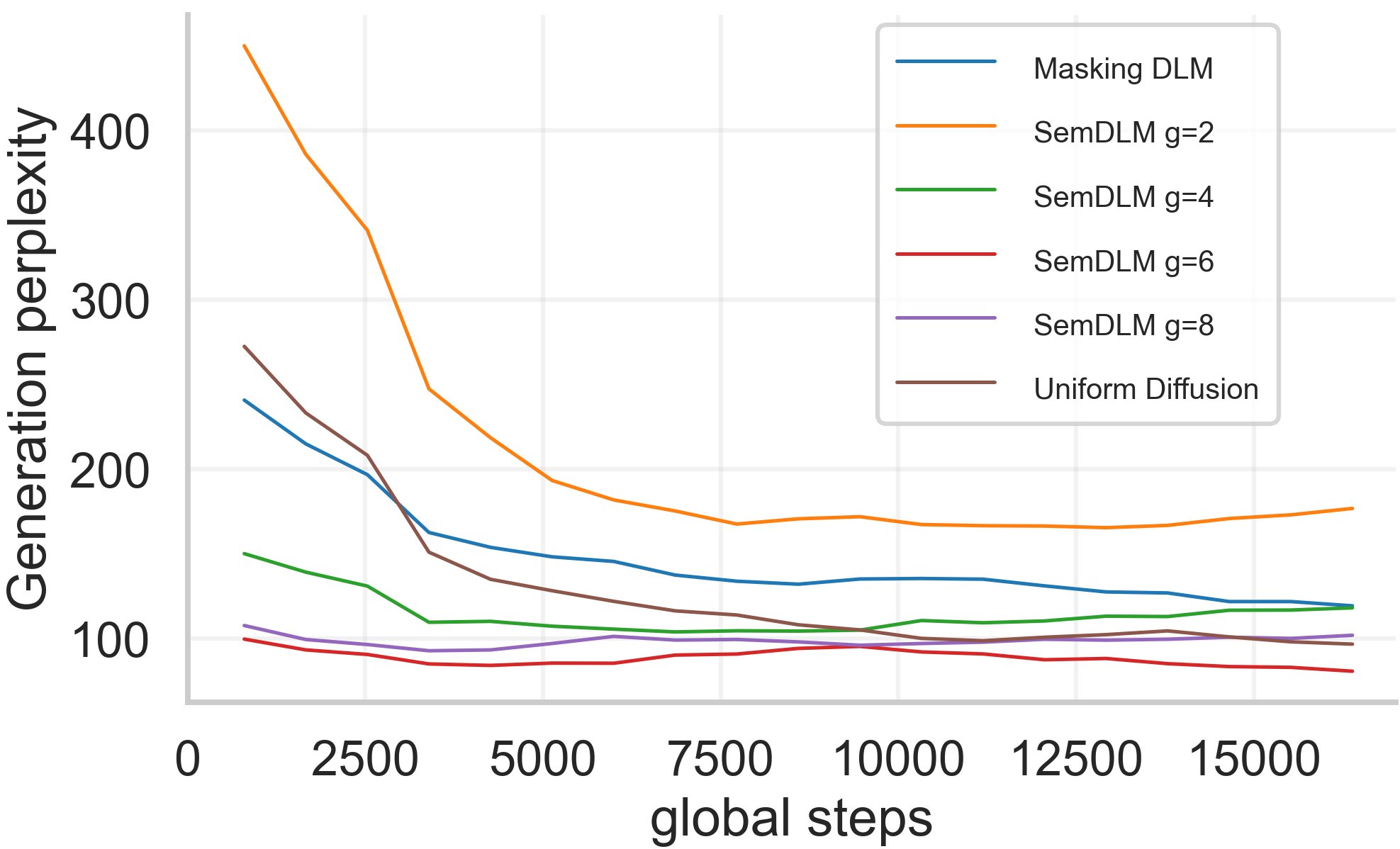}
    \captionof{figure}{Generation PPL with varying $\gamma$ during training in LM1B.}
    \label{fig:gen_ppl}
  \end{minipage}
  \hfill
  \begin{minipage}[t]{0.28\textwidth}
    \vspace{0pt}
    \centering
    \resizebox{\linewidth}{!}{
      \begin{tabular}{lc}
      \toprule
      {Model} & {PPL \(\downarrow\)} \\
      \midrule
      SemDLM (KNN)~\citep{DBLP:conf/nips/AustinJHTB21} & 149.50 \\
      Uniform DLM~\citep{DBLP:conf/nips/AustinJHTB21} & 88.03 \\
      D3PM (Masking)~\citep{DBLP:conf/nips/AustinJHTB21} & 82.34 \\
      Masking DLM~\citep{sahoo2024simple} & 31.78 \\
      SEDD (uniform)~\citep{DBLP:conf/icml/LouME24} & 40.68  \\
      SEDD (Masking)~\citep{DBLP:conf/icml/LouME24} & 32.68  \\
      GIDD~\citep{vonruette2025scalingbehaviordiscretediffusion} & 34.34 \\
      \rowcolor{Gray} SemDLM+ (Ours) & \textbf{27.19} \\
      \bottomrule
      \end{tabular}
    }
    \captionof{table}{Test PPL comparison in LM1B.}
    \label{tab:likelihood}
  \end{minipage}
\end{figure*}

\begin{table*}[!t]
\caption{Generation perplexities (PPL$\downarrow$) of models trained on LM1B and OWT. The generation length for LM1B is 128 and OWT is 1024. We believe scaling sampling steps infinitely is problematic thus constrained the sampling steps to \{128, 256, 512\} for LM1B and \{128, 256, 512, 1024\} for OWT. We conducted 3 tests and results are upper bound, and those with $^\dagger$ are reproduced by us and with $^\ddagger$ are taken from the respective paper.}
\label{tab:generation_ppl_lm1b}
\centering
\resizebox{\linewidth}{!}{
    \begin{tabular}{lcccccccc}
    \toprule
    {Model/ Sampling Steps}& \multicolumn{3}{c}{LM1B}& \multicolumn{4}{c}{Openwebtext}\\
    \cmidrule(lr){2-4} \cmidrule(lr){5-8}
     &128&256&512 &128&256&512&1024 \\
    \midrule
    D3PM (Uniform)~\citep{DBLP:conf/nips/AustinJHTB21} & /&150.37 & 133.92&407.72& 146.89& 112.86 & 96.46\\
    MDLM (no remasking)~\citep{sahoo2024simple} &93.69$^\dagger$& 85.24$^\dagger$&78.02$^\dagger$& 76.50$^\dagger$&64.20$^\dagger$& \underline{50.70}$^\dagger$& 43.00$^\dagger$ \\
    ReMDM (MDLM+remask)~\citep{wang2025remasking} & $91.19^\dagger$ & $85.21^\dagger$ & $75.98^\dagger$ & $\underline{72.61}^\dagger$ & $\textbf{62.43}^\dagger$ & $50.75^\dagger$ & $\underline{41.25}^\dagger$ \\
    SEDD (uniform)~\citep{DBLP:conf/icml/LouME24}& $117.33 ^\dagger$ & $104.51^\dagger$ & $106.06 ^\dagger$ & 153.01&150.21&163.97&52.00$^\ddagger$  \\
    SEDD (Masking)~\citep{DBLP:conf/icml/LouME24} &$\underline{74.48}^\dagger$ &$\underline{65.63}^\dagger$ &$\underline{70.74}^\dagger$ &80.51$^\dagger$&69.93$^\dagger$& 68.97$^\dagger$& 46.8$^\ddagger$ \\
    GIDD ($a=1, b=-2$)~\citep{vonruette2025scalingbehaviordiscretediffusion}   &$85.09^\dagger$ &$92.44^\dagger$ &$89.76^\dagger$ &345.53$^\dagger$&245.71$^\dagger$&106.18$^\dagger$&94.92$^\dagger$\\
    \rowcolor{Gray} SemDLM+ (Ours)&\textbf{74.52} & \textbf{54.08}& \textbf{36.60} &\textbf{64.39}&\underline{62.72}&\textbf{41.91}&\textbf{34.37}  \\
    \bottomrule
    \end{tabular}
}
\end{table*}

\begin{figure*}[t]
  \centering

  \begin{minipage}[t]{0.36\textwidth}
    \vspace{0pt}
    \centering
\includegraphics[width=\linewidth]{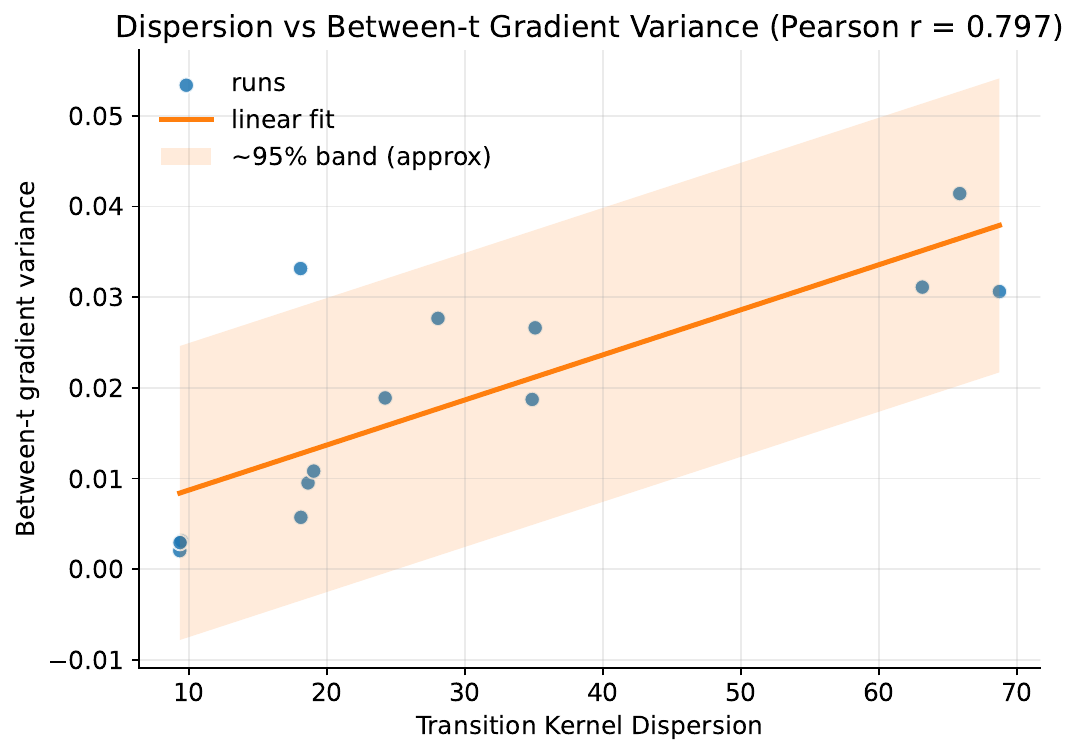}
  \caption{Observed association between transition-kernel dispersion and gradient variance in the measured settings.}
  \label{fig:variance_vs_dispersion}
  \end{minipage}
  \hfill
  \begin{minipage}[t]{0.30\textwidth}
    \vspace{0pt}
    \centering
    \includegraphics[width=\linewidth]{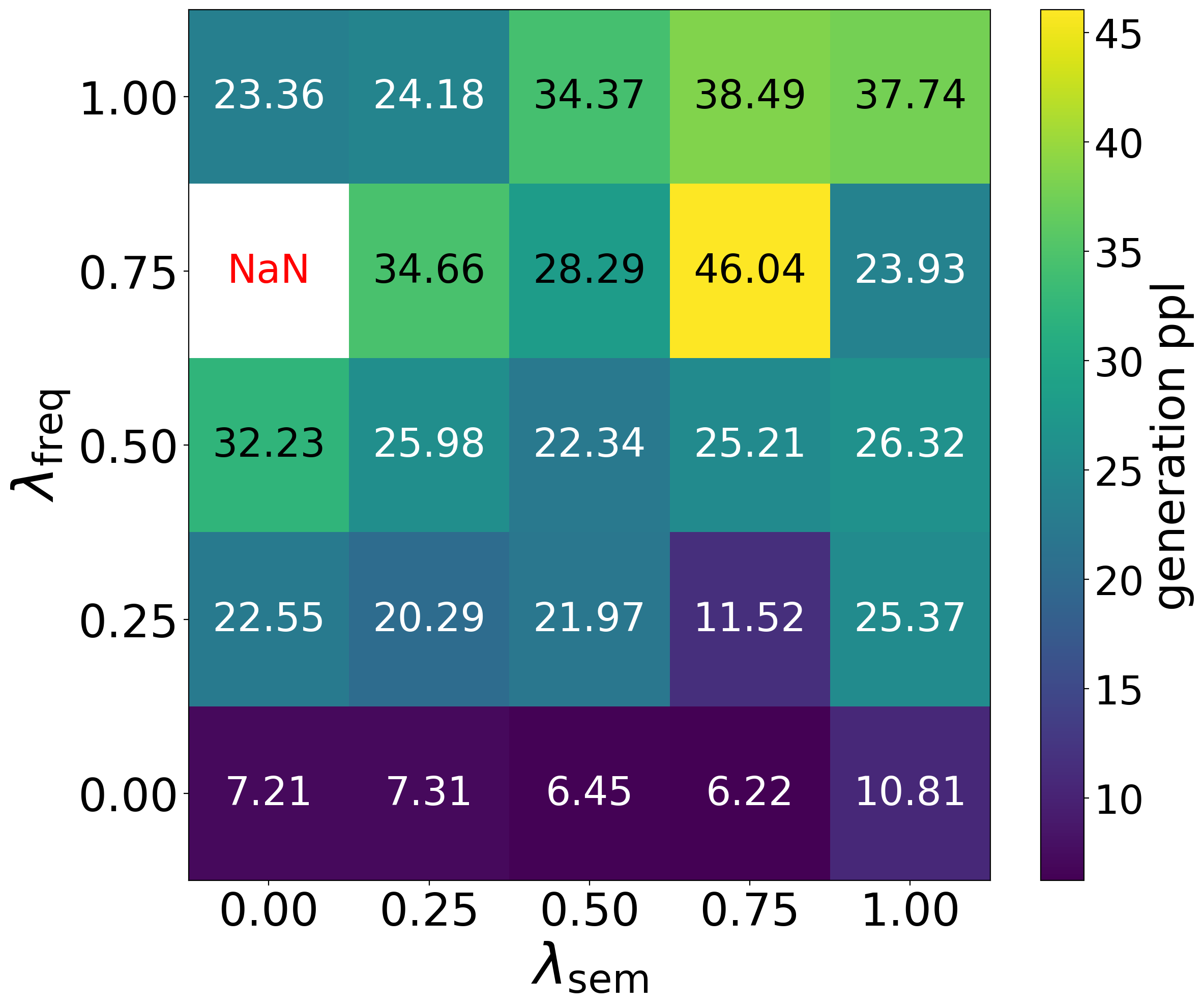}
    \captionof{figure}{Generation Perplexity with varying $\lambda_{\text{freq}}$ and $\lambda_{\text{sem}}$. SemDLM+ trained on OWT.}
    \label{fig:gen_ppl_lambda}
  \end{minipage}
  \hfill
  \begin{minipage}[t]{0.30\textwidth}
    \vspace{0pt}
    \centering
    \includegraphics[width=\linewidth]{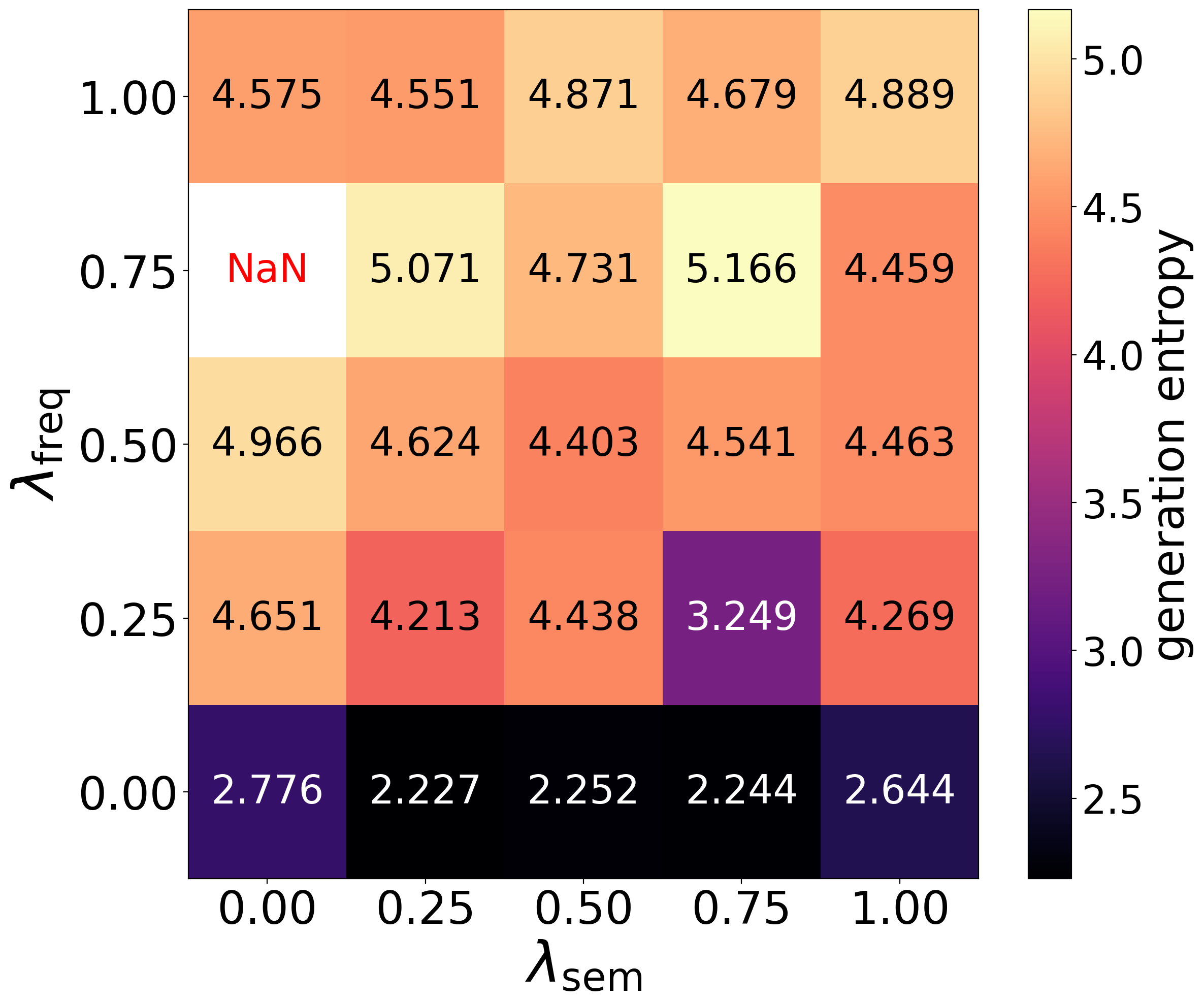}
    \captionof{figure}{Generation entropy with varying $\lambda_{\text{freq}}$ and $\lambda_{\text{sem}}$. SemDLM+ trained on OWT.}
    \label{fig:gen_entropy_lambda}
  \end{minipage}
\end{figure*}

\paragraph{Sampling Error Accumulation.}
\begin{wrapfigure}{r}{0.30\linewidth}
  \vspace{-1.2em}
  \centering
  \includegraphics[width=\linewidth]{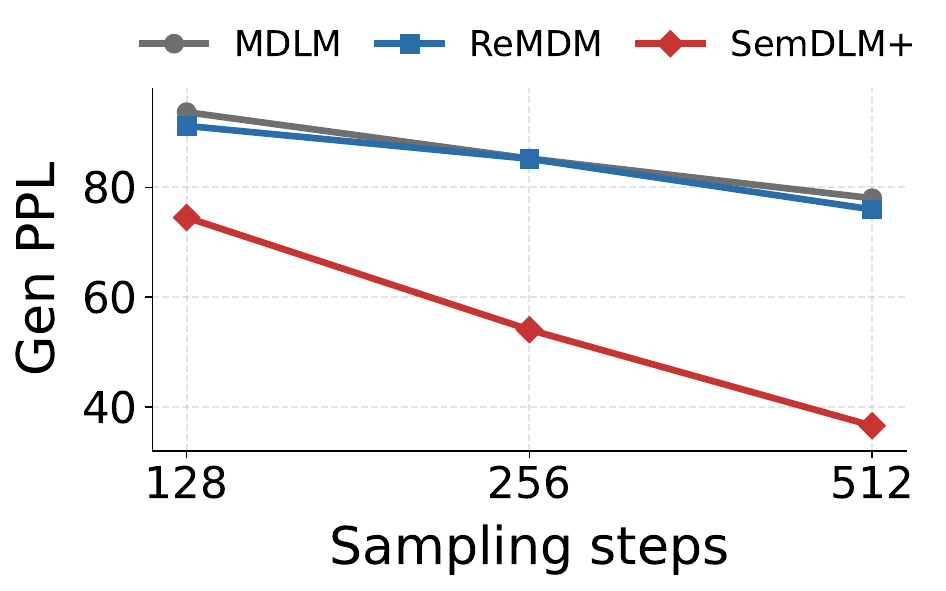}
  \caption{Gen-PPL on LM1B as sampling steps increase.}
  \label{fig:repair_curve_lm1b_wrapped}
  \vspace{-1.0em}
\end{wrapfigure}
The step-scaling results in~\cref{tab:generation_ppl_lm1b} and~\cref{fig:repair_curve_lm1b_wrapped} show how kernels handle accumulated sampling errors. SemDLM+ benefits more clearly from additional reverse steps than MDLM variants: its PPL drops from 74.52 to 36.60 on LM1B and from 64.39 to 34.37 in OWT. This suggests that SemDLM+ successfully preserves the repairing mechanism in uniform DLMs.

\section{Discussions and Limitations}

In this work, we establish an error analysis framework for understanding how transition kernel design affects DLMs and propose SemDLM+, which balances training difficulty and sampling-side repair. We hereby list important limitations for the benefit of the community. 1) Our semantic graph is token-level, so it cannot fully capture context-dependent meanings. 2) The semantic penalty is still an approximation to the ideal transition, which makes generation sensitive to the hyperparameter design. 3) Semantic basin behavior shows that generation PPL alone can be misleading. We believe broader diversity and quality evaluations are needed for the DLM community for a comprehensive analysis.

\bibliography{ref}
\bibliographystyle{plainnat} 

\appendix

\newpage
\appendix
\onecolumn

\section{Related Work}
\label{apdx:related_work}

\paragraph{Discrete Diffusion Models.}
Diffusion language models (DLMs)~\citep{austin2021structured,hoogeboom2021argmax} have emerged as a promising alternative to autoregressive (AR) language models. Unlike AR LLMs~\citep{Dubey2024TheL3,Yang2024Qwen25TR,DeepSeekAI2024DeepSeekV3TR} that decode tokens strictly left-to-right, DLMs generate text by iterative refinement, progressively denoising from completely random noise sampled from reference distribution (e.g., all [MASK] tokens or samples from uniform distribution) into a clean data points~\citep{ho2020ddpm,li2022diffusion}. DLMs have achieved competitive language modeling performance~\citep{nie2025llada, wang2026trainabilitymaskeddiffusionlanguage}, attracting interest for their potential to reduce decoding latency via parallel updates without sacrificing quality. Representative formulations include D3PM \citep{austin2021structured} and its language-model instantiations, masking-based diffusion (MDLM) \citep{sahoo2024simple,DBLP:conf/nips/ShiHWDT24,ou2025your}, as well as alternative discrete-time or continuous-time constructions such as SEDD \citep{DBLP:conf/icml/LouME24}, tau-leaping for discrete-state continuous-time diffusion \citep{DBLP:conf/nips/CampbellBBRDD22}, and discrete flow matching \citep{DBLP:conf/icml/CampbellYBRJ24}.
These models vary in their design of the \emph{transition kernel}. Masking DLMs are the predominant archetype, due to stable training and convenient likelihood objectives \citep{sahoo2024simple,DBLP:conf/nips/ShiHWDT24,ou2025your}.
Beyond masking, prior work considers uniform diffusion \citep{austin2021structured,DBLP:conf/iclr/SchiffSPWBDARPK25}, marginal diffusion \citep{qin2025defog}, and semantic-neighborhood diffusion \citep{austin2021structured} to improve mixing and address exposure to hard corruptions.
Early studies argue that masking-style diffusion can alleviate slow convergence, training instability, and weaker generalization observed in some discrete diffusion setups \citep{austin2021structured,hoogeboom2021argmax,wang202610openchallengessteering}.
More recent scaling analyses, however, report that uniform diffusion can enjoy better scaling behavior given sufficient data and training budget \citep{vonruette2025scalingbehaviordiscretediffusion,DBLP:conf/icml/RutteFDOS025, sahoo2026scalingmaskeddiffusionlanguage}.
To bridge these regimes, \emph{GIDD} \citep{DBLP:conf/icml/RutteFDOS025} interpolates between masking and uniform noising, offering a controlled mixture family for both training and inference.

\paragraph{Large Diffusion Language Models.}
At scale, \citet{nie2025llada} introduces LLaDA, the first 8B scale diffusion Large Language Models (DLLM) with masked diffusion trained \emph{from scratch} with a standard pretraining + supervised fine-tuning (SFT) pipeline. LLaDA demonstrates strong scalability and competitive in-context learning and instruction-following ability at the 8B scale. Beyond purely diffusion training from scratch, hybrid paradigms aim to combine AR coherence with diffusion parallelism: \citet{dream2025} presents Dream 7B and \citet{cheng2025sdarsynergisticdiffusionautoregressionparadigm} propose SDAR. They both convert a pretrained AR model into a block-wise diffusion decoder. Finally, Tencent~\citep{liu2025wedlmreconcilingdiffusionlanguage} target deployment efficiency by reconciling diffusion decoding with standard KV caching: they propose WeDLM, which realizes diffusion-style parallel decoding under causal attention via reordering and streaming commitment, enabling substantial speedups over optimized AR serving (e.g., vLLM). Overall, these works indicate a rapidly maturing ecosystem of large diffusion LLMs.

\paragraph{Comparing Diffusion and Autoregressive Models.}
A growing literature compares diffusion and autoregressive (AR) generation across empirical performance, architectural design, and learning/inference dynamics. Empirically,~\citet{DBLP:journals/corr/abs-2411-07873} compares AR and diffusion models on downstream tasks and characterizes practical trade-offs. From an architectural viewpoint, ~\citet{DBLP:journals/corr/abs-2505-15045} analyzes how causal (AR) versus bidirectional (diffusion-style) attention affects modeling and decoding efficiency.
On the learning side,~\citet{kim2025train} studies diffusion trajectory construction (token update ordering) and shows that it can induce a harder optimization problem than AR's fixed causal ordering; ~\citet{DBLP:conf/iclr/GongA0YZLAZB00K25} demonstrates that diffusion and AR objectives can be cast in closely related forms, enabling transfer and distillation across paradigms.
Recent scaling evidence further suggests that diffusion models may generalize better as data and compute increase \citep{DBLP:journals/corr/abs-2511-03276}.
Despite these advances, existing studies do not offer a unified framework that links training dynamics, generator transition dynamics. We serve as the first to provide a unified framework that originated from generalization error analysis, and identify the  bias-variance trade-offs in DLLMs

\paragraph{Kernel Design in Diffusion Language Models. } 
Beyond model architectures, another line of work in discrete diffusion improves decoding quality by manipulating the generation path. \citet{qin2025defog} propose target-guided diffusion to steer generation toward desired outcomes during sampling. \citet{wang2025learningorder} learns the token update order with reinforcement learning, adapting the path to the input and improving decoding efficiency. \citet{wang2025remasking} introduces remasking strategies that re-corrupt tokens during inference to correct early mistakes and reduce degeneration. From a geometric perspective, \citet{jiang2025bureswasserstein} redesigns the diffusion path via optimal-transport alignment to better match intermediate states with the target distribution. More recently, \citet{peng2025path,liu2025think} propose trainable planners that predict which positions to decode or commit at each step, enabling selective refinement and faster convergence. Collectively, these methods highlight the importance of the transition path as a controllable degree of freedom in diffusion-based sequence generation.

\paragraph{Theoretical Guarantees of DLMs.}
While diffusion language models (DLMs) are often motivated by the prospect of parallel decoding, recent theory makes precise when such acceleration is \emph{provably} attainable and when it is fundamentally limited.
\citet{Feng2025TheoreticalBA} provide a rigorous analysis for masked diffusion models (MDMs) and show that the efficiency--accuracy trade-off is highly metric-dependent.
In particular, under mild regularity assumptions, they prove that MDMs can reach near-optimal \emph{perplexity} with a number of refinement steps that does not grow with the sequence length, supporting the intuition that parallel refinement can match AR-level likelihood quality with bounded-step sampling.~\citet{kim2025train} and ~\citet{tang2026gdsdreinforcementlearningguided} demonstrate that diffusion LLMs are difficult to train, with training instability, severe training-inference mismatch, and a complex optimization landscape, but are more powerful in sampling once trained. Later,~\citet{li2025convergence_dlm} develop an information-theoretic convergence theory for DLM sampling. Their results offers a concrete theoretical lens: DLMs can admit provable acceleration for likelihood-oriented metrics and weakly dependent sequences, yet face intrinsic limitations when the goal is exact-sequence correctness or when token dependencies are strong.

\section{Omitted Proofs}
\label{apdx:omit_proofs}

\subsection{Error decomposition}
\label{apdx:bias_decomp}

\begin{lightgrey}
\begin{prop}[Error Decomposition]
The exact KL divergence in the generation process can be decomposed into:
\begin{equation}
\begin{aligned}
&D_{\mathrm{KL}}\big(q(x_0)\|p_\theta(x_0)\big)\\
=&\underbrace{\mathbb{E}_{t}\Big[D_{\mathrm{KL}}\big(q(x_t)\|p_\theta(x_t)\big)\Big]}_{\text{Sampling Error}}
+\underbrace{\mathbb{E}_{t, x_t\sim q_t}\left[
D_{\mathrm{KL}}\big(q(x_0\mid x_t)\|p_\theta(x_0\mid x_t)\big)
\right]}_{\text{Approximation Error}}
\\&
-\underbrace{\mathbb{E}_{t,x_0\sim q}\left[
D_{\mathrm{KL}} \big(q(x_t\mid x_0)\| p_\theta(x_t\mid x_0)\big)
\right]}_{\text{Forward Process Mismatch}}.
\end{aligned}
\end{equation}
1) \textbf{Approximation Error.} The middle term is the \emph{approximation (denoising) error} under the true marginals
$\{q_t\}$, which has almost the same form as our training objective in~\Cref{eq:ori_loss}. 2) \textbf{Sampling Error.} The first term is a time-averaged \emph{marginal mismatch} between the model's rolled-out
marginals $\{p_t\}$ and the true marginals $\{q_t\}$. 3) \textbf{The forward path mismatch} The last term is a time-averaged correction term depending on whether the model's forward conditionals
$p(x_t\mid x_0)$ match the true noising process $q(x_t\mid x_0)$.
\end{prop}
\end{lightgrey}

\textit{Proof.} For simplicity, in the following we omit the parameter $\theta$ but just denote $p_\theta$ as $p$. We first consider the KL divergence on the joint distribution of $(x_0, x_t)$, where
\begin{equation}
\label{eq:joint-kl-def}
D_{\mathrm{KL}}\big(q(x_0,x_t)\|p(x_0,x_t)\big) :=  \int_{x_0, x_t}dx_0 dx_t q(x_0,x_t) \log\frac{q(x_0,x_t)}{p(x_0,x_t)}.
\end{equation}
where $x_t$ and $x_0$ are sampled through $x_t\sim q(x_t\mid x_0)$, $x_0 \sim q(x_0)$.
We use \(q(x_0,x_t)=q_t(x_t) q(x_0\mid x_t)\), and \(p(x_0,x_t)=p_t(x_t) p(x_0\mid x_t)\) to write,
\begin{equation}
\log\frac{q(x_0,x_t)}{p(x_0,x_t)} = \log\frac{q_t(x_t) q(x_0\mid x_t)}{p_t(x_t) p(x_0\mid x_t)}=
\log\frac{q_t(x_t)}{p_t(x_t)} + \log\frac{q(x_0\mid x_t)}{p(x_0\mid x_t)}.
\label{eq:log-split-xt}
\end{equation}
Substitute~\Cref{eq:log-split-xt} into~\Cref{eq:joint-kl-def}, we obtain,
\begin{equation}
\label{eq:joint-kl-split-two-terms}
\begin{aligned}
&D_{\mathrm{KL}} \big(q(x_0,x_t) \| p(x_0,x_t)\big) \\&= 
\underbrace{ \int_{x_0, x_t}q(x_0,x_t)\log\frac{q_t(x_t)}{p_t(x_t)}dx_0 dx_t }_{\text{Term (1)}}
+
\underbrace{ \int_{x_0, x_t} q(x_0,x_t)\log\frac{q(x_0\mid x_t)}{p(x_0\mid x_t)}dx_0 dx_t }_{\text{Term (2)}}.
\end{aligned}
\end{equation}
Given that $\log\frac{q_t(x_t)}{p_t(x_t)}$ depends only on $x_t$ but not on $x_0$, term (1) becomes
\begin{equation}
\begin{aligned}
\text{Term (1)} &= \int_{x_t}\left(\int_{x_0} q(x_0,x_t)dx_0\right)\log\frac{q_t(x_t)}{p_t(x_t)}dx_t\\
&= \int_{x_t} q_t(x_t) \log\frac{q_t(x_t)}{p_t(x_t)} dx_t = D_{\mathrm{KL}} \big(q_t(x_t) \| p_t(x_t)\big).
\end{aligned}
\label{eq:star-equals-marginal-kl}
\end{equation}
Similarly, using $q(x_0,x_t)=q_t(x_t) q(x_0\mid x_t)$, Term (2) becomes,
\begin{equation}
\begin{aligned}
\text{Term (2)}
&= \int_{x_t, x_0} q_t(x_t) q(x_0\mid x_t)\log\frac{q(x_0\mid x_t)}{p(x_0\mid x_t)} dx_t dx_0\\
&= \int_{x_t}dx_t q_t(x_t)\left[\int_{x_0} q(x_0\mid x_t)\log\frac{q(x_0\mid x_t)}{p(x_0\mid x_t)}dx_0\right]\\
&= \int_{x_t} q_t(x_t) D_{\mathrm{KL}} \big(q(x_0\mid x_t) \| p(x_0\mid x_t)\big)dx_t\\
&= \mathbb{E}_{x_t\sim q_t} \Big[
D_{\mathrm{KL}} \big(q(x_0\mid x_t) \| p(x_0\mid x_t)\big)
\Big].
\label{eq:dagger-equals-conditional-kl}
\end{aligned}
\end{equation}
Plugging~\Cref{eq:star-equals-marginal-kl} and~\Cref{eq:dagger-equals-conditional-kl} into~\Cref{eq:joint-kl-split-two-terms} yields
\begin{equation}
\label{eq:joint-kl-chain-xt-final}
D_{\mathrm{KL}} \big(q(x_0,x_t) \| p(x_0,x_t)\big)
=
D_{\mathrm{KL}} \big(q_t \| p_t\big)
+
\mathbb{E}_{x_t\sim q_t} \Big[
D_{\mathrm{KL}} \big(q(x_0\mid x_t) \| p(x_0\mid x_t)\big)
\Big].
\end{equation}
Now consider \(q(x_0,x_t)=q(x_0) q(x_t\mid x_0)\)
 and \(p(x_0,x_t)=p(x_0) p(x_t\mid x_0)\), we get
\begin{equation}
\log\frac{q(x_0,x_t)}{p(x_0,x_t)} = \log\frac{q(x_0) q(x_t\mid x_0)}{p(x_0) p(x_t\mid x_0)}=
\log\frac{q(x_0)}{p(x_0)} + \log\frac{q(x_t\mid x_0)}{p(x_t\mid x_0)}.
\label{eq:log-split-x0}
\end{equation}
Substitute~\Cref{eq:log-split-x0} into~\Cref{eq:joint-kl-def}:
\begin{equation}
\begin{aligned}
&D_{\mathrm{KL}} \big(q(x_0,x_t) \| p(x_0,x_t)\big)
\\
&=
 \int_{x_0, x_t}dx_0 dx_t q(x_0,x_t)\left[
\log\frac{q(x_0)}{p(x_0)} + \log\frac{q(x_t\mid x_0)}{p(x_t\mid x_0)}
\right]\nonumber\\
&=
\underbrace{ \int_{x_0, x_t}dx_0 dx_t q(x_0,x_t)\log\frac{q(x_0)}{p(x_0)}}_{\text{Term (3)}}
+
\underbrace{ \int_{x_0, x_t}dx_0 dx_t q(x_0,x_t)\log\frac{q(x_t\mid x_0)}{p(x_t\mid x_0)}}_{\text{Term (4)}}.
\label{eq:joint-kl-split-two-terms-x0}
\end{aligned}
\end{equation}
Then, using the same marginalization trick as~\Cref{eq:star-equals-marginal-kl},
\begin{equation}
\begin{aligned}
\text{Term (3)} &= \int_{x_0}\left(\int_{x_t} q(x_0,x_t)dx_t\right)\log\frac{q(x_0)}{p(x_0)}dx_0\\
&= \int_{x_0} q(x_0) \log\frac{q(x_0)}{p(x_0)}dx_0 = D_{\mathrm{KL}} \big(q(x_0) \| p(x_0)\big).
\end{aligned}
\label{eq:clubsuit-equals-data-kl}
\end{equation}
And similarly for Term (4),
\begin{equation}
\begin{aligned}
\text{Term (4)}
&= \int_{x_0, x_t} dx_tdx_0\quad  q(x_0) q(x_t\mid x_0)\log\frac{q(x_t\mid x_0)}{p(x_t\mid x_0)} \\
&= \int{dx_0} q(x_0)\left[\int_{x_t} q(x_t\mid x_0)\log\frac{q(x_t\mid x_0)}{p(x_t\mid x_0)}dx_t\right]\\\
&= \mathbb{E}_{x_0\sim q} \Big[
D_{\mathrm{KL}} \big(q(x_t\mid x_0) \| p(x_t\mid x_0)\big)
\Big].
\end{aligned}
\label{eq:spadesuit-equals-forward-conditional-kl}
\end{equation}
Plugging~\Cref{eq:clubsuit-equals-data-kl} and~\Cref{eq:spadesuit-equals-forward-conditional-kl}
into~\Cref{eq:joint-kl-split-two-terms-x0} yields
\begin{equation}
\label{eq:joint-kl-chain-x0-final}
D_{\mathrm{KL}} \big(q(x_0,x_t) \| p(x_0,x_t)\big)
=
D_{\mathrm{KL}} \big(q(x_0) \| p(x_0)\big)
+
\mathbb{E}_{x_0\sim q} \Big[
D_{\mathrm{KL}} \big(q(x_t\mid x_0) \| p(x_t\mid x_0)\big)
\Big].
\end{equation}
Both~\Cref{eq:joint-kl-chain-xt-final} and~\Cref{eq:joint-kl-chain-x0-final} equal the same quantity
$D_{\mathrm{KL}}(q(x_0,x_t) \| p(x_0,x_t))$, then
\begin{equation}
\begin{aligned}
&D_{\mathrm{KL}} \big(q_t(x_t) \| p_t(x_t)\big)
+
\mathbb{E}_{x_t\sim q(x_t)} \Big[
D_{\mathrm{KL}} \big(q(x_0\mid x_t) \| p(x_0\mid x_t)\big)
\Big]
\\=
&D_{\mathrm{KL}} \big(q(x_0) \| p(x_0)\big)
+
\mathbb{E}_{x_0\sim q(x_0)} \Big[
D_{\mathrm{KL}} \big(q(x_t\mid x_0) \| p(x_t\mid x_0)\big)
\Big].\nonumber
\label{eq:equate-two-forms}
\end{aligned}
\end{equation}
Hence
\begin{equation}
\label{eq:fixed-t-three-term}
\begin{aligned}
&D_{\mathrm{KL}} \big(q(x_0) \| p(x_0)\big) \\
= &D_{\mathrm{KL}} \big(q_t \| p_t\big)
+
\mathbb{E}_{x_t\sim q_t} \Big[
D_{\mathrm{KL}} \big(q(x_0\mid x_t) \| p(x_0\mid x_t)\big)
\Big]
-
\mathbb{E}_{x_0\sim q} \Big[
D_{\mathrm{KL}} \big(q(x_t\mid x_0) \| p(x_t\mid x_0)\big)
\Big].
\end{aligned}
\end{equation}
Taking expectation of~\Cref{eq:fixed-t-three-term} with respect to $t$ yields,
\begin{lightgrey}
\begin{equation}
\begin{aligned}
D_{\mathrm{KL}} \big(q(x_0) \| p(x_0)\big)
= &\mathbb{E}_{t} \Big[D_{\mathrm{KL}} \big(q_t \| p_t\big)\Big]
+
\mathbb{E}_{t, x_t\sim q_t} \Big[
D_{\mathrm{KL}} \big(q(x_0\mid x_t) \| p(x_0\mid x_t)\big)
\Big] 
\\&-
\mathbb{E}_{t,x_0\sim q} \Big[
D_{\mathrm{KL}} \big(q(x_t\mid x_0) \| p(x_t\mid x_0)\big)
\Big].
\end{aligned}
\label{eq:three-term-time-avg}
\end{equation}
\end{lightgrey}
which is exactly the three-term error decomposition that we are interested.

\subsection{When the forward process matches}
Assume that for all $(x_0,t)$,\(p(x_t\mid x_0)\equiv q(x_t\mid x_0).\)
Then for every $(x_0,t)$, \(D_{\mathrm{KL}} \big(q(x_t\mid x_0) \| p(x_t\mid x_0)\big)=0\). Hence the entire time-averaged correction term in~\Cref{eq:three-term-time-avg} is zero:
\begin{equation}
\mathbb{E}_{t, x_0\sim q} \Big[
D_{\mathrm{KL}} \big(q(x_t\mid x_0) \| p(x_t\mid x_0)\big)
\Big]=0.
\end{equation}
Therefore~\Cref{eq:three-term-time-avg} reduces to
\begin{equation}
D_{\mathrm{KL}} \big(q(x_0) \| p(x_0)\big) =\underbrace{
\mathbb{E}_{t} \Big[D_{\mathrm{KL}} \big(q_t \| p_t\big)\Big]}_{\text{Sampling Error}}
+
\underbrace{\mathbb{E}_{t,x_t\sim q_t} \Big[
D_{\mathrm{KL}} \big(q(x_0\mid x_t) \| p(x_0\mid x_t)\big)
\Big]}_{\text{Approximation Error}}.
\end{equation}
This equation is importance, as it helps us to consider the error seperately from sampling and training.

\subsection{Bias-Variance Decomposition}
\label{apdx:bvd}

\subsubsection{Training and sampling bias-variance trade-off}
The bias-variance decomposition holds for both training and sampling, so we utilize $p(x_0\mid x_t)$ in approximation error as an illustration, which can be easily generalized to sampling distribution $p(x_t)$.

With the expected predictor
\begin{equation}
\bar p(x_0\mid x_t)
:= \mathbb E_S\big[p_{\hat \theta_S}(x_0\mid x_t)\big] =
\mathbb E_S\big[\hat p_S(x_0\mid x_t)\big].
\end{equation}
where $\hat \theta(S)$ denote parameter induced by specific training set and optimization randomness $S$ (short as $\hat p_S := p_{\hat \theta_S}$). 
Then~\Cref{eq:ori_loss} admits the following exact decomposition ($p(x_0\mid x_t)$ simplified as $p$):
\begin{equation}
\begin{aligned}
&\mathbb E_S\Big[ D_{\mathrm{KL}}(q \| \hat p_S)\Big]
=
D_{\mathrm{KL}}(q \| \bar p) + \gV(x_t),
\\
& \text{where, }\gV(x_t)
:=
\mathbb E_{q(x_0\mid x_t)}
\Big[
\log \bar p -
\mathbb E_S \log (\hat p_S)
\Big]
\ge 0.
\end{aligned}
\end{equation}
The nonnegativity of $\gV(x_t)$ follows from Jensen's inequality:
$\mathbb E_S[\log \hat p_S(x_0\mid x_t)] \le \log \mathbb E_S[\hat p_S(x_0\mid x_t)]
= \log \bar p(x_0\mid x_t)$. Specifically, 
\begin{equation}
\begin{aligned}
&\mathbb{E}_S \Big[ D_{\mathrm{KL}}\big(q(x_0\mid x_t)  \|  p_{\hat{\theta}_S}(x_0\mid x_t)\big) \Big] = \mathbb{E}_S \left[ \int q(x_0\mid x_t) \log \frac{q(x_0\mid x_t)}{p_{\hat{\theta}_S}(x_0\mid x_t)}   dx_0 \right] \\
&= \int q(x_0\mid x_t) \log q(x_0\mid x_t)   dx_0 - \int q(x_0\mid x_t) \underbrace{\mathbb{E}_S \left[ \log p_{\hat{\theta}_S}(x_0\mid x_t) \right]}_{\log \bar{p}(x_0)}   dx_0 \\
&= \int q(x_0\mid x_t) \left( \log q(x_0\mid x_t) - \log \bar{p}(x_0\mid x_t) \right)   dx_0 \\
&\quad + \int q(x_0\mid x_t) \left( \log \bar{p}(x_0\mid x_t) - \mathbb{E}_S [\log p_{\hat{\theta}_S}(x_0\mid x_t)] \right)   dx_0 \quad \text{(Add \& Sub)} \\
&= \underbrace{D_{\mathrm{KL}}(q  \|  \bar{p})}_{\textbf{Bias}} + \underbrace{\mathbb{E}_{q(x_0\mid x_t)}
\left[
\log \bar p -
\mathbb E_S \log (\hat p_S)
\right]}_{\textbf{Variance}}.
\end{aligned}
\end{equation}
Averaging over $x_t$ yields an expected risk decomposition:
\begin{equation}
\mathbb{E}_{t,x_t}\mathbb E_S\Big[ D_{\mathrm{KL}}(q(x_0\mid x_t) \| \hat p_S(x_0\mid x_t))\Big]
=
\underbrace{
\mathbb E_{t, x_t}\Big[
D_{\mathrm{KL}}\big(q(x_0\mid x_t)\| \bar p(x_0\mid x_t)\big)
\Big]}_{\textbf{Asymptotic Bias } \gB_{\text{asym}}}
+
\underbrace{
\mathbb E_{t, x_t}\big[\mathcal V(x_t)\big]}_{\textbf{Variance } \gV},
\end{equation}
where $x_t$ are distributed according to the forward process, $q(x_t)=\int_{x_0}q(x_t\mid x_0)p(x_0)dx_0$.
The sampling bias-variance trade-off follows the exact same derivation, which only requires replacing \(p(x_0\mid x_t)\) with \(p(x_t)\), such that
\begin{equation}
D_{\mathrm{KL}}\big(q_t(x_t)\|p_t(x_t)\big) = \underbrace{D_{\mathrm{KL}}\big(q_t(x_t)\|\bar{p}_t(x_t)\big)}_{\textbf{Exposure Bias }\gB_{t}} + \underbrace{\mathbb{E}_{x_t}[\log \bar{p}(x_t) - \mathbb{E}_S (\log \hat{p}_S(x_t))]}_{\textbf{Sampling Roll-out Variance }\gV_t}
\end{equation}

\subsection{Derivation for Asymptotic Bias: Local Geometry and Posterior Approximation Difficulty}
\label{apdx:asym_bias}

This section, we prove Proposition ~\ref{prop:appx_diff}, which we restate here
\begin{lightgrey}
\begin{prop}[Approximation difficulty across diffusion variants.]
We derive in~\Cref{apdx:asym_bias} that, 
\begin{equation}
\mathbb E \left[d_{\mathrm{act}}^{\mathrm{mask}}\right]
\leq
\mathbb E \left[d_{\mathrm{act}}^{\mathrm{sem}}\right]
\mathbb E 
\leq\left[d_{\mathrm{act}}^{\mathrm{uni}}\right], \quad
\mathbb E \left[\gI_1^{\mathrm{mask}}\right]
\leq
\mathbb E \left[\gI_1^{\mathrm{sem}}\right]
\leq
\mathbb E \left[\gI_1^{\mathrm{uni}}\right]
\end{equation}
\end{prop}
\end{lightgrey}

\subsubsection{Characterizing the local geometry with Hessian matrix}

To describe the posterior approximation difficulty, we characterize the \emph{local geometry} of the target posterior. Let \(f_\theta(x_t)^i \in \mathbb R^{|V|}\) denote the predicted logits at position \(i\in[1:L]\) such that \(p_\theta^i(\cdot\mid x_t)=\softmax(f_\theta(x_t)^i)\).
Assuming conditional factorization across positions given \(x_t\), the asymptotic bias can be decomposed into local losses:
\begin{equation}
\gB_{\text{asym}}
=
\mathbb E_{t,x_t} \left[
\sum_{i=1}^L
\ell_i \left(f_\theta(x_t)^i; x_t\right)
\right],
\qquad
\ell_i(f_\theta;x_t):=
D_{\mathrm{KL}}\bigl(q(x_0^i\mid x_t)\|\softmax(f_\theta)\bigr).
\label{eq:appendix-global-kl-local-decomp}
\end{equation}

Fix a position \(i\) and a corrupted \(x_t\), we have,
\begin{equation}
\ell_i(f;x_t)
=
D_{\mathrm{KL}}\bigl(q(x_0^i=\cdot\mid x_t)\,\|\,\softmax(f_\theta)\bigr)
=
\sum_{k\in V} q(x_0^i=k\mid x_t)
\log
\frac{q(x_0^i=k\mid x_t)}{p_\theta(x_0^i=k\mid x_t)}.
\label{eq:appendix-local-objective-explicit}
\end{equation}
Since \(\sum_k q(x_0^i=k\mid x_t)\log q(x_0^i=k\mid x_t)\) is constant w.r.t \(\theta\), minimizing \(\ell_i(f;x_t)\) is equivalent to minimizing the cross-entropy
\[
-\sum_{k\in V} q(x_0^i=k\mid x_t)\log p_\theta(x_0^i=k\mid x_t).
\]
Recall that \(\log p_\theta(x_0^i=k\mid x_t)
=
f_\theta(x_t)^i_k - \log \sum_{m\in V} e^{f_\theta(x_t)^i_m}\), when \(f_\theta(x_t)^i_k\) denote the \(k\)-th logit at position \(i\) (we use $f_{k}$ when there is no ambuguity), we obtain
\[
\frac{\partial \log p_\theta(x_0^i=k\mid x_t)}{\partial f_{r}}
=
\delta_{kr}-p_\theta(x_0^i=r\mid x_t).
\]
Therefore, the gradient is
\begin{equation}
\begin{aligned}
\frac{\partial \ell_i(f;x_t)}{\partial f_{r}}
&=
-\sum_{k\in V} q(x_0^i=k\mid x_t) \frac{\partial \log p_\theta(x_0^i=k\mid x_t)}{\partial f_{\theta,r}}\\
&=
-\sum_{k\in V} q(x_0^i=k\mid x_t)(\delta_{kr}-p_\theta(x_0^i=r\mid x_t))\\
&=
-q(x_0^i=r\mid x_t) + p_\theta(x_0^i=r\mid x_t) \underbrace{\sum_{k\in V} q(x_0^i=k\mid x_t)}_{=1}
\\&=
p_\theta(x_0^i=r\mid x_t)-q(x_0^i=r\mid x_t),\nonumber
\end{aligned}
\label{eq:first_order_gradient}
\end{equation}

\begin{lightgrey}
By the chain rule and the decomposition of the global loss,
\begin{align}
\nabla_\theta \gL
&=
\nabla_\theta
\mathbb E_{t,x_t}\!\left[
\sum_{i=1}^L \ell_i\!\left(f_\theta(x_t)^i;x_t\right)
\right]
\nonumber\\
&=
\mathbb E_{t,x_t}
\left[
\sum_{i=1}^L
\sum_{k\in V}
\frac{\partial \ell_i(f_\theta(x_t)^i;x_t)}{\partial f_\theta(x_t)^i_k}
\cdot
\nabla_\theta f_\theta(x_t)^i_k
\right]
\nonumber\\
&=
\mathbb E_{t,x_t}
\left[
\sum_{i=1}^L
\sum_{k\in V}
\Big(
p_\theta(x_0^i=k\mid x_t)-q(x_0^i=k\mid x_t)
\Big)
\cdot
\nabla_\theta f_\theta(x_t)^i_k
\right],
\end{align}
which is exactly~\Cref{eq:param-grad-local-mismatch}.
\end{lightgrey}
Since \(q\) is irrelevant with f, differentiate~\Cref{eq:first_order_gradient} once more we obtain 
\[
\nabla_f^2 \ell(f;q)=\frac{\partial p}{\partial f}.
\]
For softmax, the Jacobian is
\[
\frac{\partial p_k}{\partial f_r}
=
p_k(\delta_{kr}-p_r),
\]
hence
\begin{equation}
\nabla_f^2 \ell(f;q)
=
\text{Diag}(p)-pp^\top.
\label{eq:appendix-hessian-general}
\end{equation}
At any optimum \(f^\star(q)\) such that \(\softmax(f^\star(q))=q\), we obtain
\begin{equation}
\nabla_f^2 \ell(f^\star(q);q)
=
\text{Diag}(q)-qq^\top
=
\Sigma(q).
\label{eq:appendix-hessian-opt}
\end{equation}

\paragraph{Local Taylor expansion.}
By standard second-order expansion around \(f^\star(q)\), and denote \(\delta f \approx f-f^{*}(q)\) a small variation around \(f\), using taylor expansion around \(f^*\) we have,
\begin{equation}
\ell(f^\star(q)+\delta f;q)
=
\ell(f^\star(q);q)
+
\underbrace{\nabla_f \ell(f^\star(q);q)}_{=\,0}^\top \delta f
+
\frac12 \delta f^\top \Sigma(q)\delta f
+
o(\|\delta f\|^2),
\end{equation}
which gives
\begin{equation}
\ell(f;q)-\ell(f^\star(q);q)
=
\frac12 \delta f^\top \Sigma(q)\delta f
+
o(\|\delta f\|^2).
\end{equation}
hence, we obtain~\Cref{eq:local-grad-linearization}, where
\begin{equation}
\nabla_f \ell(f;q)
=
\Sigma(q)(f- f^*(q))
+o \left(\|f - f^*(q)\|\right), \text{with }\Sigma(q):=\nabla^2_f\ell(f;q)=\text{Diag}(q)-qq^\top.
\end{equation}
This suggest that the shape of the second order hessian will dominant the optimization dynamics of our diffusion LLMs.

\subsubsection{Metrics for posterior approximation difficulty.}

Through our Hessian matrix \(\Sigma(q)\), we can then define a few metrics that is helpful in analyzing the posterior approximation difficulty in DLMs. Specifically,
\begin{equation}
d_{\mathrm{act}}(q) := \operatorname{rank}\bigl(\Sigma(q)\bigr),\quad
\gI_1(q) := \operatorname{tr}\bigl(\Sigma(q)\bigr),
\quad
\gI_2(q)
:=
\operatorname{tr}\bigl(\Sigma(q)^2\bigr).
\label{eq:appendix-difficulty-defs}
\end{equation}
Specifically,
\begin{enumerate}
    \item \(d_{\mathrm{act}}\) counts the dimension of the tangent space on which the KL objective has non-zero curvature. Hence it is the number of independent logit directions that must be fitted simultaneously.
    \item \(\gI_1\) measures the average curvature and therefore how much an isotropic local logit error increases the KL loss.
    \item \(\gI_2\) measures the squared curvature and therefore the local gradient energy induced by the same logit error.
\end{enumerate}

We provide some justification on what they mean in the optimization problem. 

\paragraph{\(d_{\mathrm{act}}(q)\) measures the independent direction to be fitted.}

Given that 
\(d_{\mathrm{act}}(q) = \operatorname{rank}(\Sigma(q) = \operatorname{supp}(q)\). It directly measure the independent directions that is required to be fitted.

\paragraph{2) \(\gI_1(q)\) measures local error sensitivity.}
Using \(\operatorname{tr}(qq^\top)=\|q\|_2^2\),
\begin{equation}
\gI_1(q)
=
\operatorname{tr}(\Sigma(q))
=
\operatorname{tr}(\operatorname{Diag}(q))-\operatorname{tr}(qq^\top)
=
1-\|q\|_2^2.
\label{eq:appendix-i1-trace}
\end{equation}
Around the optimal logit \(f^\star(q)\),
\(
\mathbb E[\delta f]=0\) and \(\mathbb E[\delta f\delta f^\top]=\sigma^2 I \), The Taylor expansion gives
\begin{equation}
\mathbb E\bigl[\ell(f^\star(q)+\delta f;q)-\ell(f^\star(q);q)\bigr]
=
\frac{\sigma^2}{2}\operatorname{tr}(\Sigma(q))
+o(\sigma^2)
=
\frac{\sigma^2}{2}\gI_1(q)+o(\sigma^2).
\label{eq:appendix-i1-loss-sensitivity}
\end{equation}
Therefore \(\gI_1(q)\) is exactly the first-order coefficient converting local logit mismatch into excess posterior KL.
\paragraph{\(\gI_2(q)\) measures gradient energy.}
The local gradient linearization gives \(\nabla_f\ell(f^\star(q)+\delta f;q) = \Sigma(q)\delta f+o(\|\delta f\|)\).
Hence
\begin{equation}
\mathbb E\|\nabla_f\ell(f^\star(q)+\delta f;q)\|_2^2
=
\sigma^2\operatorname{tr}(\Sigma(q)^2)+o(\sigma^2)
=
\sigma^2\gI_2(q)+o(\sigma^2).
\label{eq:appendix-i2-gradient-energy}
\end{equation}
Thus \(\gI_2(q)\) measures the local gradient energy created by posterior fitting. A larger \(\gI_2\) means that small logit errors induce larger gradients and can contribute to noisier optimization.

\begin{lightgrey}
Together, \(d_{\mathrm{act}}\), \(\gI_1\), and \(\gI_2\) summarize three aspects of posterior fitting: how many logit directions are active, how strongly local mismatch becomes KL error, and how much gradient energy the mismatch creates. Smaller values indicate an easier local posterior approximation problem.
\end{lightgrey}

\paragraph{Approximation difficulty across diffusion variants.}
We next show how the transition kernel changes these quantities. For simplicity, we fix one position and suppress the position index and the context. Let \(q_{\mathrm{data}}(x_0=k)\) be the clean-token data prior, and define the support as
\(S:=\{k:q_{\mathrm{data}}(x_0=k)>0\}\).We denote the forward kernel as \(q_t(j\mid k):=q(x_t=j\mid x_0=k)\) and gives the posterior
\begin{equation}
q(x_0=k\mid x_t=j)
=
\frac{q_{\mathrm{data}}(x_0=k)q_t(j\mid k)}
{\sum_{r\in S}q_{\mathrm{data}}(x_0=r)q_t(j\mid r)}.
\label{eq:appendix-bayes-local-kernel}
\end{equation}
Thus the support for the posterior is,
\begin{equation}
\operatorname{supp} q(x_0=\cdot\mid x_t=j)
=
S\cap \{k:q_t(j\mid k)>0\}.
\label{eq:appendix-posterior-support-kernel}
\end{equation}
This candidate set controls \(d_{\mathrm{act}}\), while its concentration controls \(\gI_1=1-\|q\|_2^2\).

We then discuss different diffusion variants. Please note that all the following discussion we use \(q_t(j\mid k):=q(x_t=j\mid x_0=k)\) for simplicity.

\paragraph{Masking diffusion.}
For the absorbing kernel,
\[
q_t^{\mathrm{mask}}(j\mid k)
=
\alpha_t\mathbb I\{j=k\}
+(1-\alpha_t)\mathbb I\{j=\mathtt{[MASK]}\}.
\]
If \(j\neq\mathtt{[MASK]}\), the clean token must be \(j\), so the posterior is \(\delta_j\) and
\(d_{\mathrm{act}}=\gI_1=0\). If \(j=\mathtt{[MASK]}\), the corrupted token gives no token-level information and the posterior is \(q_{\mathrm{data}}(x_0=\cdot)\), so \(d_{\mathrm{act}}=|S|-1\) and \(\gI_1=1-\sum_{k\in S}q_{\mathrm{data}}(x_0=k)^2\). Since the mask event has probability \(1-\alpha_t\), we can take expectation and get
\begin{equation}
\mathbb E[d_{\mathrm{act}}^{\mathrm{mask}}]
=
(1-\alpha_t)(|S|-1),
\qquad
\mathbb E[\gI_1^{\mathrm{mask}}]
=
(1-\alpha_t)\left(1-\sum_{k\in S}q_{\mathrm{data}}(x_0=k)^2\right).
\label{eq:appendix-mask-difficulty}
\end{equation}

\paragraph{Uniform diffusion.}
For the uniform kernel,
\[
q_t^{\mathrm{uni}}(j\mid k)
=
\alpha_t\mathbb I\{j=k\}
+|\mathcal V|^{-1}(1-\alpha_t).
\]
When \(0<\alpha_t<1\), \(q_t^{\mathrm{uni}}(j\mid k)>0\) for every \(k\in S\). Hence every observed \(j\) leaves all clean tokens in \(S\) possible:
\begin{equation}
d_{\mathrm{act}}^{\mathrm{uni}}(j)=|S|-1.
\label{eq:appendix-uni-dact}
\end{equation}
Uniform corruption is therefore harder than masking with more active dimension. It is also less informative such that
\begin{equation}
\mathbb E[\gI_1^{\mathrm{mask}}]
\le
\mathbb E[\gI_1^{\mathrm{uni}}].
\label{eq:appendix-mask-uni-i1}
\end{equation}

\paragraph{Semantic diffusion.}
For a localized semantic kernel,
\[
q_t^{\mathrm{sem}}(j\mid k)
=
\alpha_t\mathbb I\{j=k\}
+(1-\alpha_t)s_t^{\mathrm{sem}}(j\mid k),
\]
the possible clean explanations of \(j\) are
\[
S_{\mathrm{sem}}(j)
:=
S\cap
\left(
\{j\}\cup
\{k:s_t^{\mathrm{sem}}(j\mid k)>0\}
\right).
\]
Then
\begin{equation}
d_{\mathrm{act}}^{\mathrm{sem}}(j)
=
|S_{\mathrm{sem}}(j)|-1.
\label{eq:appendix-sem-dact}
\end{equation}
If the semantic neighborhoods are genuinely local, then
\[
\{j\}\cap S
\subseteq
S_{\mathrm{sem}}(j)
\subseteq
S,
\]
so semantic diffusion activates a local set of clean explanations: larger than a copied masking token, but smaller than the full clean support used by uniform diffusion. This gives
\begin{equation}
\mathbb E[d_{\mathrm{act}}^{\mathrm{mask}}]
\le
\mathbb E[d_{\mathrm{act}}^{\mathrm{sem}}]
\le
\mathbb E[d_{\mathrm{act}}^{\mathrm{uni}}].
\label{eq:appendix-dact-order-context}
\end{equation}

For \(\gI_1\), recall that \(\gI_1(q)=1-\|q\|_2^2\) increases as the posterior becomes less concentrated. Masking gives delta posteriors on visible tokens, uniform keeps the broadest set of explanations active, and semantic diffusion lies between them by restricting explanations to semantic neighborhoods. Therefore, under the same intermediate-kernel condition,
\begin{equation}
\mathbb E[\gI_1^{\mathrm{mask}}]
\le
\mathbb E[\gI_1^{\mathrm{sem}}]
\le
\mathbb E[\gI_1^{\mathrm{uni}}].
\label{eq:appendix-i1-order-context}
\end{equation}
A simple sufficient case is a locally flat posterior over each candidate set: if \(q\) is approximately uniform over \(m\) candidates, then \(\gI_1(q)=1-1/m\), which grows with \(m\).

Finally, averaging \eqref{eq:appendix-dact-order-context} and \eqref{eq:appendix-i1-order-context} over \(t\) yields the main-text comparison:
\begin{equation}
\mathbb E \left[d_{\mathrm{act}}^{\mathrm{mask}}\right]
\le
\mathbb E \left[d_{\mathrm{act}}^{\mathrm{sem}}\right]
\le
\mathbb E \left[d_{\mathrm{act}}^{\mathrm{uni}}\right],
\qquad
\mathbb E \left[\gI_1^{\mathrm{mask}}\right]
\le
\mathbb E \left[\gI_1^{\mathrm{sem}}\right]
\le
\mathbb E \left[\gI_1^{\mathrm{uni}}\right].
\label{eq:appendix-main-difficulty-order}
\end{equation}

\paragraph{Dense semantic proposals.}
If \(s_t^{\mathrm{sem}}\) is implemented by a dense softmax over semantic scores, then the exact support can become the full vocabulary, making the exact rank equal to the uniform rank. In that case the same argument should be read in terms of the \(\epsilon\)-effective support
\[
\operatorname{supp}_\epsilon(q):=\{k:q_k\ge\epsilon\},
\qquad
d_{\mathrm{act},\epsilon}(q):=|\operatorname{supp}_\epsilon(q)|-1,
\]
or equivalently in the low-temperature regime where most posterior mass lies in the semantic neighborhood. This is the practical sense in which semantic diffusion reduces the active posterior geometry relative to uniform diffusion while retaining more active directions than masking.
\subsection{Proof for Exposure-bias Propagation}
\label{apdx:exposure_bias_proof}

We restate proposition~\ref{prop:exposure_bias} and give proof here.
\begin{lightgrey}
\begin{prop}[Exposure Bias Propagation]
With mild conditions, the propagation satisfies:
\begin{equation}
\gB_{t-1}
\le
\eta_t \,\gB_t + \rho_t,\nonumber
\end{equation}
where \(\rho_t\) upper-bounds the step-wise error:
\( \rho_t \geq \sup_{x_t}
D_{\mathrm{KL}}\bigl(q_t(x_{t-1}\mid x_t) \| p_t^\theta(x_{t-1}\mid x_t)\bigr) \) and \(\eta_t^{\mathrm{KL}}\in[0,1]\) is the step-wise error propagation coefficient. In particular,
\begin{equation}
\eta_{t,\mathrm{mask}} = 1,
\quad
\eta_{t,\mathrm{rm}}
\le
1-r_t\lambda_t^{\mathrm{ref}},
\quad
\eta_{t,\mathrm{uni}}
\le
1-\lambda_t^{\mathrm{uni}}<1,
\quad
\eta_{t,\mathrm{sem}}
\le
1-\lambda_t^{\mathrm{sem}}<1. \nonumber
\end{equation}
Consequently, for \textbf{semantic and uniform diffusion}, \(\rho_t\le \rho\) and \(\eta_t \le \eta<1\), then 
\[\mathcal{B}_{\exp} = \sum_{t=0}^{T}\gB_t = \mathcal{O}\!\left(\frac{T\,\rho}{1-\eta}\right);\]
whereas for \textbf{masking diffusion} \(\eta_t^{\mathrm{KL}}\approx 1\) and 
\[\mathcal{B}_{\exp} = \mathcal{O}\!\left(T^2 \rho\right).\]
\end{prop}
\end{lightgrey}

Please note that, in the main text, we use \(t-dt\) instead of \(t-1\) in the recursion. This is just a re-indexing of the \(t\) from exact continuous time to sampling step \(T\), thus have no impact in our conclusion. In this proof, we follow \(t=[0:T]\) notation as this can simplify the proof. 

\textit{Proof.} For one reverse step \(t\to t-1\), write
\begin{equation}
Q_t(y\mid x)
:=
q(x_{t-1}=y\mid x_t=x),
\qquad
P_t^\theta(y\mid x)
:=
p_\theta(x_{t-1}=y\mid x_t=x).
\end{equation}
Thus \(q_{t-1}=q_tQ_t\), \(p_{t-1}=p_tP_t^\theta\), and
\(\gB_t=D_{\mathrm{KL}}(q_t\|p_t)\).  We measure the repair ability of the
true reverse kernel by its KL contraction coefficient
\begin{equation}
\eta_t
:=
\sup_{\mu,\nu}
\frac{D_{\mathrm{KL}}(\mu Q_t\|\pi Q_t)}
{D_{\mathrm{KL}}(\mu\|\pi)}
\in[0,1],
\label{eq:eta-kl-appendix}
\end{equation}
where the supremum is over pairs with finite nonzero denominator. \(\eta_t\) measures under true transition kernel \(Q_t\), what will be the error propagated. 

We then use \(\rho_t\) to measure the error between true kernel \(Q_t\) and learned kernel \(P_t^\theta\), 
\[
\rho_t
\ge
\sup_{x,y}
\log
\frac{Q_t(y\mid x)}{P_t^\theta(y\mid x)}.
\]
This gives 
\begin{equation}
Q_t(y\mid x)\le e^{\rho_t}P_t^\theta(y\mid x),
\qquad \forall x,y.
\label{eq:kernel-domination-appendix}
\end{equation}
With the definition of \(\eta_t\) and \(\rho_t\), we can derive the one-step exposure bias recursion. 
\begin{theorem}[One-step exposure-bias recursion]
\label{thm:kl-exposure-bias}
Under~\Cref{eq:eta-kl-appendix} and \Cref{eq:kernel-domination-appendix}, we have
\begin{equation}
\gB_{t-1}
\le
\eta_t\,\gB_t+\rho_t .
\label{eq:linear-kl-recursion-appendix}
\end{equation}
\end{theorem}

\begin{proof}
We expand the per-step exposure bias as
\[
\gB_{t-1}
=D_{\mathrm{KL}}(q_{t-1}\|p_{t-1}^\theta) \text{ and }\gB_{t-1}
=D_{\mathrm{KL}}(q_t\|p_t^\theta)
\]
By \Cref{eq:kernel-domination-appendix}, the mixed distributions also satisfy
\[
(p_tQ_t)(y)
=\sum_xp_t(x)Q_t(y\mid x)
\le
e^{\rho_t}\sum_xp_t(x)P_t^\theta(y\mid x)
=e^{\rho_t}(p_tP_t^\theta)(y).
\]
Therefore, with \(q_{t-1} = q_tQ_t\) and \(p_{t-1} = p_t P_t^\theta\) 
\begin{align}
D_{\mathrm{KL}}(q_{t-1}\|p_{t-1}^\theta) = D_{\mathrm{KL}}(q_tQ_t\|p_tP_t^\theta)
&=
\sum_y(q_tQ_t)(y)
\log\frac{(q_tQ_t)(y)}{(p_tP_t^\theta)(y)}
\nonumber\\
&\le
\sum_y(q_tQ_t)(y)
\log\frac{(q_tQ_t)(y)}{(p_tQ_t)(y)}
+\rho_t
\nonumber\\
&=
D_{\mathrm{KL}}(q_tQ_t\|p_tQ_t)+\rho_t
\nonumber\\
&\le
\eta_tD_{\mathrm{KL}}(q_t\|p_t^\theta)+\rho_t .
\end{align}
This is exactly \Cref{eq:linear-kl-recursion-appendix}.
\end{proof}
The recursion shows that exposure bias has two sources.  The term \(\rho_t\)
is the local denoising error of this step.  The coefficient \(\eta_t\) decides
whether previous rollout mismatch is preserved or contracted.

\paragraph{When is \(\eta_t<1\)?}
The useful case is when the reverse kernel contains a shared component (for instance, global jumping).  If
there exist \(\lambda_t>0\) and a distribution \(\nu_t\) such that
\begin{equation}
Q_t(\cdot\mid x)\ge \lambda_t\nu_t(\cdot)\quad\forall x,\text{ such that }\eta_t\le 1-\lambda_t
\quad 
\label{eq:minorization-appendix}
\end{equation}
We can \Cref{eq:minorization-appendix} allows us to rewrite the transition
\[
Q_t(\cdot\mid x)
=
\lambda_t\nu_t(\cdot)+(1-\lambda_t)\widetilde Q_t(\cdot\mid x)
\]
for another Markov kernel \(\widetilde Q_t\). This is because for any \(\mu,\nu\), joint
convexity of KL gives
\[
D_{\mathrm{KL}}(\mu Q_t\|\nu Q_t)
\le
(1-\lambda_t)
D_{\mathrm{KL}}(\mu\widetilde Q_t\|\nu\widetilde Q_t),
\]
and ordinary data processing gives
\[
D_{\mathrm{KL}}(\mu\widetilde Q_t\|\nu\widetilde Q_t)
\le
D_{\mathrm{KL}}(\mu\|\nu).
\]
Thus, a source-independent
refresh mass (for instance, the global jumping in uniform and semantic diffusion) can make the reverse step reduce exposure bias.

\paragraph{Instantiating the coefficient for diffusion variants.}
We now apply the above criterion to the transition kernels discussed in the
main text.  Let \(a_t\in(0,1)\) denote the one-step probability of retaining
the current token in the forward kernel.

\textbf{Masking diffusion.}
For the absorbing kernel, observing a visible token \(x_t=j\neq\mathtt{[MASK]}\)
forces \(x_{t-1}=j\).  Hence, on the visible-token subspace, \(Q_t^{\mathrm{mask}}(\cdot\mid j)=\delta_j\).
The reverse step acts as the identity on that subspace, so one can choose two
visible-token distributions \(\mu,\nu\) for which
\[
D_{\mathrm{KL}}(\mu Q_t^{\mathrm{mask}}\|\nu Q_t^{\mathrm{mask}})
=
D_{\mathrm{KL}}(\mu\|\nu).
\]
Therefore \(\eta_{t,\mathrm{mask}}=1\). This reflects the ``early commitment'' effect: once a wrong visible token appears, the standard masking reverse kernel has no mechanism to repair it.

\textbf{Remasking.}
If a remasking sampler has the effective kernel
\(Q_t^{\mathrm{rm}}=(1-r_t)Q_t^{\mathrm{mask}}+r_tQ_t^{\mathrm{ref}}\), 
and the refresh part satisfies
\[Q_t^{\mathrm{ref}}(\cdot\mid x)\ge \lambda_t^{\mathrm{ref}}\nu_t(\cdot),
\qquad \forall x,
\]
then \(Q_t^{\mathrm{rm}}(\cdot\mid x)\ge
r_t\lambda_t^{\mathrm{ref}}\nu_t(\cdot)\). Hence \(\eta_{t,\mathrm{rm}} \le
1-r_t\lambda_t^{\mathrm{ref}}\).
Remasking therefore helps precisely through the amount of shared refresh it
injects.

\textbf{Uniform diffusion.}
For a one-step uniform forward kernel on a vocabulary of size \(V\),
\(
q(x_t=j\mid x_{t-1}=y) = a_t\delta_{jy}+V^{-1}(1-a_t).
\)
Bayes' rule gives
\[
Q_t^{\mathrm{uni}}(y\mid j)
=
\frac{
\left(a_t\delta_{jy}+\frac{1-a_t}{V}\right)q_{t-1}(y)
}{
q_t(j)
}
\ge
\frac{1-a_t}{Vq_t(j)}q_{t-1}(y).
\]
Thus \(Q_t^{\mathrm{uni}}\) has the shared component \(q_{t-1}\) with mass
\[
\lambda_t^{\mathrm{uni}}
:=
\min_j\frac{1-a_t}{Vq_t(j)} ,
\]
assuming \(q_t(j)>0\).  By~\Cref{eq:minorization-appendix},
\begin{equation}
\eta_{t,\mathrm{uni}}
\le
1-\lambda_t^{\mathrm{uni}}
<1.
\label{eq:uniform-eta-appendix}
\end{equation}
Uniform diffusion can therefore contract rollout mismatch, which is the
sampling-side repair advantage emphasized in the main text.

\textbf{Semantic diffusion and SemDLM+.}
For a semantic forward kernel
\[
q(x_t=j\mid x_{t-1}=y)
=
a_t\delta_{jy}+(1-a_t)s_t^{\mathrm{sem}}(j\mid y),
\]
similarly, if the semantic diffusion is equiped with a global transition \(\nu_t\) as a shared component:
\[
s_t^{\mathrm{sem}}(j\mid y)\ge \alpha_t^{\mathrm{sem}}\nu_t(j),
\qquad \forall y,j .
\]
Then
\[
Q_t^{\mathrm{sem}}(y\mid j)
\ge
\frac{(1-a_t)\alpha_t^{\mathrm{sem}}\nu_t(j)}
{q_t(j)}
q_{t-1}(y),
\]
so
\begin{equation}
\eta_{t,\mathrm{sem}}
\le
1-\lambda_t^{\mathrm{sem}}
<1,
\qquad
\lambda_t^{\mathrm{sem}}
:=
\min_j
\frac{(1-a_t)\alpha_t^{\mathrm{sem}}\nu_t(j)}
{q_t(j)}.
\label{eq:semantic-eta-appendix}
\end{equation}
This serves another important factor that we need global transition in~\Cref{eq:global_jump_kernel}.

\paragraph{Cumulated Exposure Bias with Propagation.}
We then extend the one-step analysis to the overall analysis of exposure bias. 

Assume \(\rho_t\le\rho\) for all \(t\) and \(\eta_t\le\eta<1\) (this is the case for semantic and uniform diffusion), then iterating
\(\gB_{t-1}\le\eta\gB_t+\rho\) from time \(t = [0:T]\) gives
\begin{equation}
\gB_t
\le
\eta^{T-t}\gB_T
+
\frac{1-\eta^{T-t}}{1-\eta}\rho .
\label{eq:Bt-contracting-appendix}
\end{equation}
Summing over all the time steps \(t=0,\ldots,T\) yields
\begin{equation}
\sum_{t=0}^{T}\gB_t
\le
\frac{\gB_T}{1-\eta}
+
\frac{T\rho}{1-\eta}.
\label{eq:accumulated-contracting-appendix}
\end{equation}
If \(q_T=p_T\), we get \(\gB_T=0\), and therefore
\[
\sum_{t=0}^{T}\gB_t
=
\mathcal O\!\left(\frac{T\rho}{1-\eta}\right).
\]

For masking diffusion, \(\eta_t=1\).  The same recursion becomes
\(\gB_{t-1}\le \gB_t+\rho\), so
\[
\gB_t\le \gB_T+(T-t)\rho
\]
and hence
\begin{equation}
\sum_{t=0}^{T}\gB_t
\le
(T+1)\gB_T+\frac{T(T+1)}{2}\rho .
\label{eq:accumulated-masking-appendix}
\end{equation}
When \(q_T=p_T\), this reduces to
\[
\sum_{t=0}^{T}\gB_t
=
\mathcal O(T^2\rho).
\]
This proves the comparison in Proposition~\ref{prop:exposure_bias}: kernels with shared support, such as uniform diffusion and SemDLM+ with global jumping, contract
exposure bias across steps, while masking can accumulate it quadratically in the number of reverse steps.

\subsection{Transition Kernel Dispersion and Variance}

\subsubsection{Connecting and unifying sampling and approximation variance}
\label{apdx:unify_ts_variance}

As pointed out in~\cref{sec:error_analysis_framework}, we will show that sampling variance can be 
controlled by the approximation variance, such that in the main text we only consider the approximation variance in analysis. 

\begin{lightgrey}
\textbf{Final-sample risk and posterior approximation.}
To recall, the rollout risk $\gR$ the rollout risk
at the final time $t=0$, where $\hat p_S(x_0)$ is the distribution of
the generated samples. We can then show the expected risk can be upper bounded by the approximation approximation error,
\begin{equation}
\begin{aligned}
&\mathbb E_S[\gR(\hat\theta_S)]
=\mathbb E_S D_{\mathrm{KL}}(q(x_0)\|\hat p_S(x_0))
=\gB_0+\gV_0\\
&\le \underbrace{D_{\mathrm{KL}}(q(x_T)\|p(x_T))}
_{\text{Initial mismatch}} \quad+\sum_{t=1}^T\mathbb E_{q}
\Big[
\underbrace{D_{\mathrm{KL}}(q(x_0\mid x_t)\|\bar p(x_0\mid x_t))}
_{\text{Approximation bias}}
+\underbrace{\gV(x_t)}_{\text{Approximation variance}}
\Big].
\end{aligned}
\label{eq:final-sample-risk-bv}
\end{equation}
Thus cumulative approximation bias and variance control an
upper bound on the risk. Thus, the above inequality makes explicit why we study approximation variance in the main text.
\end{lightgrey}

\textit{Proof:}

We follow the notation of the main text: With optimization randomness $S$,
$\hat p_S(x_0\mid x_t)$ denotes the learned clean-data predictor, 
$\hat p_S(x_t)$ the actual rollout marginal, and $\bar p$ their respective
expectations over training-data.

Consider a fixed finite sampling grid $t=0,\ldots,T$, using the same trained
model $S$ at every reverse step. The $x_0$-prediction parameterization gives
\begin{equation}
\hat p_S(x_{t-1}\mid x_t)
=\sum_{x_0}q(x_{t-1}\mid x_t,x_0)\hat p_S(x_0\mid x_t).
\label{eq:variance-reverse-kernel}
\end{equation}
Here $q(x_{t-1}\mid x_t,x_0)$ is pre-specified forward process. 
With
\[
\hat p_S(x_{t-1})
=\sum_{x_t}\hat p_S(x_t)\hat p_S(x_{t-1}\mid x_t).
\]
and assuming a common initial distribution $\hat p_S(x_T)=p(x_T)$ leads to $\bar p(x_0\mid x_t):=\mathbb E_S\hat p_S(x_0\mid x_t)$ and
$\bar p(x_t):=\mathbb E_S\hat p_S(x_t)$. Recall the defined approximation variance $\gV(x_t)$ and the
rollout quantities $\gB_t,\gV_t$:
\begin{align}
\gV(x_t)
&:=\mathbb E_{q(x_0\mid x_t)}
 [\log\bar p(x_0\mid x_t)-\mathbb E_S\log\hat p_S(x_0\mid x_t)],\\
\gB_t
&:=D_{\mathrm{KL}}(q(x_t)\|\bar p(x_t)),\\
\gV_t
&:=\mathbb E_{q(x_t)}
 [\log\bar p(x_t)-\mathbb E_S\log\hat p_S(x_t)].
\end{align}
The per-step approximation variance is $\mathbb E_{q(x_t)}[\gV(x_t)]$,
whereas $\gV_t$ is the per-step rollout variance. We can then have the folliwng proposition

\begin{lightgrey}
\begin{prop}[Rollout risk controlled by approximation error]
\label{prop:cumulative-rollout-variance}
For every $t\in\{0,\ldots,T\}$,
\begin{equation}
\begin{aligned}
\gB_t+\gV_t
\le{}&D_{\mathrm{KL}}(q(x_T)\|p(x_T))+\sum_{s=t+1}^{T}\mathbb E_{q(x_s)}
\Big[
D_{\mathrm{KL}}(q(x_0\mid x_s)\|\bar p(x_0\mid x_s))
+\gV(x_s)
\Big].
\end{aligned}
\label{eq:cumulative-rollout-bv}
\end{equation}
If additionally $p(x_T)=q(x_T)$ and
$\bar p(x_0\mid x_s)=q(x_0\mid x_s)$ for $q(x_s)$-almost every $x_s$
and every $s>t$, then
\begin{equation}
\gV_t\le\gB_t+\gV_t
\le\sum_{s=t+1}^{T}\mathbb E_{q(x_s)}[\gV(x_s)].
\label{eq:calibrated-rollout-variance}
\end{equation}
\end{prop}
\end{lightgrey}

\begin{proof}
For each fixed $x_t$, we have proved in the previous sections that
\[
D_{\mathrm{KL}}(q(x_{t-1}\mid x_t)\|\hat p_S(x_{t-1}\mid x_t))
\le D_{\mathrm{KL}}(q(x_0\mid x_t)\|\hat p_S(x_0\mid x_t)).
\]
where the equality holds when the forward process is assumed to be equivalemnt. Then the joint distribution
$q(x_t)q(x_{t-1}\mid x_t)$ and
$\hat p_S(x_t)\hat p_S(x_{t-1}\mid x_t)$ have marginals
$q(x_{t-1})$ and $\hat p_S(x_{t-1})$, respectively.
Therefore, marginalization and the KL chain rule therefore yield
\begin{align}
&D_{\mathrm{KL}}(q(x_{t-1})\|\hat p_S(x_{t-1}))\nonumber\\
&\quad\le D_{\mathrm{KL}}(q(x_t)\|\hat p_S(x_t))
+\mathbb E_{q(x_t)}
 D_{\mathrm{KL}}(q(x_{t-1}\mid x_t)\|\hat p_S(x_{t-1}\mid x_t))
\nonumber\\
&\quad\le D_{\mathrm{KL}}(q(x_t)\|\hat p_S(x_t))
+\mathbb E_{q(x_t)}
 D_{\mathrm{KL}}(q(x_0\mid x_t)\|\hat p_S(x_0\mid x_t)).
\label{eq:variance-one-step}
\end{align}
Iterating from $T$ to $t$ and taking $\mathbb E_S$ gives
\[
\begin{aligned}
\mathbb E_S D_{\mathrm{KL}}(q(x_t)\|\hat p_S(x_t))
\le{}&D_{\mathrm{KL}}(q(x_T)\|p(x_T))\\
&+\sum_{s=t+1}^{T}\mathbb E_{q(x_s)}\mathbb E_S
D_{\mathrm{KL}}(q(x_0\mid x_s)\|\hat p_S(x_0\mid x_s)).
\end{aligned}
\]
The bias--variance decompositions in the main text identify the left-hand
side as $\gB_t+\gV_t$ and each summand as
\[
\mathbb E_{q(x_s)}
D_{\mathrm{KL}}(q(x_0\mid x_s)\|\bar p(x_0\mid x_s))
+\mathbb E_{q(x_s)}[\gV(x_s)].
\]
This proves the general bound.
\end{proof}

With such a bound, we motivate an analysis on approximation variance (and this removes the necessaries in analyzing sampling variance) 
as a training-side contribution to rollout risk.

\subsubsection{The gradient guided variance}
\label{apdx:gradient_variance}
To connect the variance term to a more classical estimation-variance picture, we rewrite the KL divergence as the combination of negative log-likelihood and entropy.
\begin{equation}
    \begin{aligned} 
    \gL &= -\mathbb{E}_{x\sim p_\text{data}(x)} \left[\log p_\theta(x) \right] - \text{Entropy}
    \end{aligned}
\end{equation}
And define the per-sample loss $\ell_t(\theta) := -\log p_\theta(x_0\mid x_t)$, per-step risks $\gL_t(\theta):=\mathbb E_{x_t}[\ell_t(\theta)]$ and use $\nabla_\theta$ as the gradient operator. Let $\theta^\star$ minimize the population risk
$\gL(\theta)=\mathbb E_t[\mathcal L_t(\theta)]$, so that $\mathbb E_t[g_t(\theta^\star)]=0$. At $\theta^\star$,
the law of total variance yields
\begin{equation}
\begin{aligned}
&\mathrm{Var}_{t, x_t}\big(\nabla \ell_t(\theta^*)\big)\\
=&\mathbb E_t\left[\mathrm{Var}_{x_t}\big(\nabla \ell_t(\theta^*)\big)\right]
+ \mathrm{Var}_t\left(\mathbb E_{x_t} \left[\nabla \ell_t(\theta^*)\right]\right),\\
=&\underbrace{\mathbb E_t\left[\mathrm{Var}_{x_t}\big(\nabla \ell_t(\theta^*)\big)\right]}_{\text{within-step noise}}
+
\underbrace{\mathrm{Var}_t\left(\nabla \gL_t(\theta^\star)\right)}_{\text{between-step heterogeneity}}.
\end{aligned}
\end{equation}
The second term $\mathrm{Var}_t\left(\nabla \gL_t(\theta^\star)\right)$ is a \emph{kernel-induced heterogeneity}. It quantifies how the gradient varies across different diffusion steps $t$. 

\begin{lightgrey}
    \textbf{Remark.} Crucially, $\nabla \gL_t(\theta^\star)$ depends on $Q_t$ via sampling $x_t$, hence a larger dispersion of $Q_t$ will induce larger between-t heterogeneity. Other than the theoretical results,~\cref{fig:variance_vs_dispersion} empirically valid that the gradient variance is strongly correlated with the dispersion of the transition kernel across time t.
\end{lightgrey}

Now we aim to rigorously derive that as Dispersion $Q_t$ increases, Between-$t$ Heterogeneity increases.

We define the between-$t$ heterogeneity as:
\begin{equation}
    \gV(Q) \coloneqq \mathrm{Var}_t\big( \nabla \gL_t(\theta^\star; Q) \big) = \mathbb{E}_t\big[ \| \nabla \gL_t(\theta^\star; Q) \|^2 \big],
\end{equation}
where the equality holds because $\mathbb{E}_t[\nabla \gL_t(\theta^\star; Q)] = \mathbf{0}$.

We define the posterior entropy as $h(t; Q) \coloneqq \mathbb{E}_{x_t \sim q_t(\cdot; Q_t)} \big[ H\big( q(x_0 \mid x_t; Q_t) \big) \big]$, where $H(\cdot)$ is Shannon entropy. Then we define the kernel dispersion as $\mathcal{D}(Q) \coloneqq \mathrm{Var}_t\big( h(t; Q) \big) = \mathbb{E}_t\big[ h(t; Q)^2 \big] - \big( \mathbb{E}_t[h(t; Q)] \big)^2.$ This quantifies how posterior uncertainty fluctuates across diffusion steps. Uniform diffusion yields large $\mathcal{D}(Q)$ (entropy spans $[0, \log|\gV|]$); semantic diffusion yields small $\mathcal{D}(Q)$ (entropy concentrated near $\log k$), and Masking diffusion has the lowest dispersion.

We make the following assumptions:

1) the approximation error scales with posterior entropy, such that there exists $c_1 > 0$ such that for all $t$,
\begin{equation}
\mathbb{E}_{x_t} \big[ \| p_{\theta^\star}(\cdot \mid x_t) - q(\cdot \mid x_t; Q_t) \|_1 \big] \geq c_1 \cdot h(t; Q).
\end{equation}
This is induced by the fact that high-entropy posteriors (induced by dispersed $Q_t$) are harder to approximate with finite-capacity models

2) The gradients are bounded and stable.
$\| \nabla_\theta f_{\theta^\star}(x_t) \| \leq M$ for all $x_t$, and the direction of $\nabla_\theta f_{\theta^\star}(x_t)$ varies smoothly over semantically coherent neighborhoods.

\begin{lightgrey}
\begin{prop}
\label{prop:dispersion_q_gradient}
There exists a constant $c = c_1 / M > 0$ such that:
\begin{equation}
    \gV(Q) \geq c^2 \cdot \mathcal{D}(Q).
\end{equation}
Consequently, if two generators satisfy $\mathcal{D}(Q^A) > \mathcal{D}(Q^B)$, then $\gV(Q^A) > \gV(Q^B)$.
\end{prop}
\end{lightgrey}

\textit{proof:} We previously derived the gradient derivation for cross-entropy loss, and we re-formula it here. (in~\Cref{apdx:asym_bias})
\begin{align}
\begin{aligned}
        \nabla \gL_t(\theta^\star; Q) 
    &= \mathbb{E}_{x_t} \Big[ \mathbb{E}_{x_0 \mid x_t} \big[ (p_{\theta^\star}(x_0 \mid x_t) - q(x_0 \mid x_t; Q_t)) \odot \nabla_\theta f_{\theta^\star}(x_t) \big] \Big] \\
    &= \mathbb{E}_{x_t} \big[ \bm{\delta}_t(x_t; Q) \big],
\end{aligned}
\end{align}
where $\ell_t(x_t; Q)$ denotes the per-sample gradient signal. By Hölder's inequality:
\begin{align}
\begin{aligned}
    \| \nabla \gL_t(\theta^\star; Q) \| 
    &\geq \big| \mathbb{E}_{x_t} \big[ \| \bm{\delta}_t(x_t; Q) \| \big] \big| \\
    &\geq c_1 \cdot \mathbb{E}_{x_t} \big[ H(q(x_0 \mid x_t; Q_t)) \big] \cdot \frac{1}{M} \\
    &= c \cdot h(t; Q), \quad c \coloneqq c_1 / M > 0.
\end{aligned}
\end{align}
Squaring both sides and taking expectation over $t$:
\begin{align}
\begin{aligned}
\gV(Q) 
    &= \mathbb{E}_t \big[ \| \nabla \gL_t(\theta^\star; Q) \|^2 \big] \\
    &\geq c^2 \cdot \mathbb{E}_t \big[ h(t; Q)^2 \big] \\
    &= c^2 \cdot \Big( \mathrm{Var}_t\big( h(t; Q) \big) + \big( \mathbb{E}_t[h(t; Q)] \big)^2 \Big) \\
    &\geq c^2 \cdot \mathrm{Var}_t\big( h(t; Q) \big) \\
    &= c^2 \cdot \mathcal{D}(Q).
\end{aligned}
\end{align}
The directly completes the proof.

\textbf{Remark}: $\mathcal{D}(Q)$ explicitly captures how $Q_t$'s structural dispersion, such as support size, entropy trajectory, propagates to gradient statistics. This suggests that minimizing kernel dispersion $\mathcal{D}(Q)$ during design can reduce between-$t$ gradient heterogeneity. This inspires us to design SemDLM.

\section{More about Diffusion as CTMC}

In~\Cref{sec:preliminaries}, we mentioned that the local transition \(q_{t\mid t-dt}\), the cumulative forward \(q_{t\mid 0}\), and the generator $Q_t$ are equivalent representations of the same forward process. The equivalence is obtained via Bayesian rule: \(q\left(x_{t-dt} \mid x_t, x_0\right) \propto q\left(x_t \mid x_{t-dt}\right) q(x_{t-dt} \mid x_0)\). 

We thus present other details for the process. We first detail the forward process and reverse process in CTMC. The forward process is derived by the \textit{infinitesimal generator} $Q_t$, i.e., \(\frac{d q_t}{d t}=Q_t q_t,0\leq t\leq 1\), where
\begin{equation}
\text{Forward Process via Euler Sampling: }q\left(x_{t+dt}=y \mid x_t=z\right)=\delta_{z y}+Q_t(z, y) dt+O\left(dt^2\right)
\end{equation}
and the corresponding reverse process is 
\[
\text{Reverse-time Euler Step: }q(x_{t-dt}=z\mid x_t=y)
=
\delta_{yz}+\bar Q_t(y,z)dt+O(dt^2),
\]
With the reverse-time infinitesimal generator
\[
\bar Q_t(y,z)
=
Q_t(z,y)\frac{q_t(z)}{q_t(y)},\quad z\neq y,
\qquad
\bar Q_t(y,y)=-\sum_{z\neq y}\bar Q_t(y,z).
\]
where the density ratio \(q_t(z)/q_t(y)\) is typically approximated by a learned score or posterior model in practice.
 
Thus, it is clear that the local transition \(q_{t\mid t-dt}\), the cumulative forward \(q_{t\mid 0}\), and the generator $Q_t$ are equivalent representations of the same forward/reverse process. For completeness, we list the other representation of masking, uniform and semantic diffusion.

\textbf{Absorbing (Mask) Transition} defines an absorbing token $\mathtt{[MASK]}$ to make $Q_t(\mathtt{[MASK]},y)=0$ and 
\begin{equation}
Q_t(z,y)=
\begin{cases}
\lambda(t), & z\neq \mathtt{[MASK]}, y=\mathtt{[MASK]},\\
-\lambda(t), & z\neq \mathtt{[MASK]},y=z.
\end{cases}
\end{equation} 
where $\lambda(t)$ is the time schedule. This induces a cumulated kernel directly from $x_0$ to $x_t$:
\begin{equation}
    q(x_t = j \mid x_0 = i) = \alpha_t  \delta_{ij} + (1 - \alpha_t)  \delta_{j, \mathtt{[MASK]}},
\end{equation}
where $\alpha_t = \exp\!\left(- \int_0^t \lambda(s) ds\right)$ and $\delta_{ij}$ is the Kronecker delta.

\textbf{Uniform Transition} is another choice that has been shown to have favorable scaling behavior~\citep{vonruette2025scalingbehaviordiscretediffusion}:
\begin{equation}
Q_t(z,y)=
\begin{cases}
\frac{\lambda(t)}{K}, & y\neq z,\\
-\lambda(t)\big(1-\frac{1}{K}\big), & y=z.
\end{cases}
\end{equation}
which has the following cumulated kernel:
\begin{equation}
    q(x_t = j \mid x_0 = i) = \alpha_t  \delta_{ij} + \frac{1 - \alpha_t}{K}, \quad \forall  i, j \in \mathcal{V},
\end{equation}
with the same decay factor $\alpha_t = \exp\!\left(- \int_0^t \lambda(s) ds\right)$.

\textbf{Semantic Transition} Let $\mathrm{KNN}(z)$ be the $k$ nearest neighbors of $z$ and $w(z,y)\ge 0$ be edge weights, the transition kernel is then,
\begin{equation}
Q_t(z,y)=
\begin{cases}
\lambda(t)\frac{w(z,y)}{\sum_{y'\in \mathrm{KNN}(z)} w(z,y')}, & y\in \mathrm{KNN}(z),\\
-\lambda(t), & y=z,\\
0, & \text{otherwise}.
\end{cases}
\end{equation}
This gives a forward kernel as,
\begin{equation}
    q(x_t = j \mid x_0 = i)
    =
    \alpha_t \delta_{ij}+
    \frac{1-\alpha_t}{k_t}
    \mathbb I\bigl(j\in \mathcal N_{k}(i)\bigr),
\end{equation}
where $\mathcal N_{k}(i)$ is the top-$k$ semantic neighborhood of token $i$. 

\section{Additional Design for Semantic Diffusion}

\begin{minipage}[t]{0.48\textwidth}
\begin{algorithm}[H] 
\caption{Training of SemDLM}
\label{alg:semdlm-training}
\begin{algorithmic}[1]
\Require Dataset $\mathcal{D} \sim p_0$
\Ensure Trained model $f_\theta(x_t,t)$
\State Initialize parameters $\theta$
\While{not converged}
    \State Sample a minibatch $\{x_0^{(i)}\}_{i=1}^B \sim \mathcal{D}$
    \State Sample $t \sim \mathcal{U}(0,1)$
    \State Sample noisy states $x_t^{(i)} \sim p(x_t \mid x_0^{(i)})$ via Eq.~(\ref{eq:global_jump_kernel})
    \State Compute $p_{0\mid t}^\theta(\cdot \mid x_t^{(i)}) = f_\theta(x_t^{(i)}, t)$
    \State Update $\theta$ by minimizing Eq.~(\ref{eq:ori_loss})
\EndWhile
\end{algorithmic}
\end{algorithm}
\end{minipage}
\hfill 
\begin{minipage}[t]{0.48\textwidth}
\begin{algorithm}[H] 
\caption{Sampling from SemDLM}
\label{alg:semdlm-sampling}
\begin{algorithmic}[1]
\Require Reference distribution $p_1$, trained model $f_\theta(x_t,t)$, step size $\Delta t$
\Ensure Generated sample $\hat{x}_0$
\State Sample initial state $\hat{x}_1 \sim p_1$
\For{$t = 1, 1-\Delta t, \dots, \Delta t$}
    \State Compute $p_{0\mid t}^\theta(\cdot \mid \hat{x}_t) = f_\theta(\hat{x}_t,t)$
    \State Sample $\tilde{x}_0 \sim p_{0\mid t}^\theta(\cdot \mid \hat{x}_t)$
    \State Compute generator $Q_\theta(\hat{x}_t \mid \tilde{x}_0)$
    \State Sample $\hat{x}_{t-\Delta t} \sim \hat{x}_t + Q_\theta(\hat{x}_t \mid \tilde{x}_0)\,\Delta t$
\EndFor
\State \Return $\hat{x}_0$
\end{algorithmic}
\end{algorithm}
\end{minipage}

\subsection{Analysis of Semantic Basins}
\label{apdx:semantic_basin_analysis}

\paragraph{Posterior logit decomposition.}
Again, for a fixed
position \(i\), let the current sampling state be \(\hat x_t\), and write
\(\hat x_t^i=j\). The ideal denoising posterior satisfies
\begin{equation}
q(x_0^i=k\mid \hat x_t)
\propto
p_{\mathrm{data}}(x_0^i=k\mid \hat x_t^{-i})
q_t(j\mid x_0^i=k,\hat x_t^{-i}).
\label{eq:posterior_bayes_appendix}
\end{equation}
Thus, up to a normalization constant, the ideal logits \(l_i^\star(k;\hat x_t)\)
\begin{equation}
l_i^\star(k;\hat x_t)
=
\underbrace{
\log p_{\mathrm{data}}(x_0^i=k\mid \hat x_t^{-i})
}_{\text{contextual prior}}
+
\underbrace{
\log q_t(\hat x_t^i\mid x_0^i=k,\hat x_t^{-i})
}_{\text{local forward likelihood}}
+
\mathrm{const}.
\end{equation}
The first term is the contextual prior, while the second term is the local
forward likelihood. During sampling, the context \(\hat x_t^{-i}\) is generated
by the model itself. Hence the actual model logit can be written as
\begin{equation}
l_{\theta,i}(k;\hat x_t)
=
l_i^\star(k;\hat x_t)
+
\Delta_{\mathrm{roll},i}(k;\hat x_t)
+
\epsilon_i(k),
\label{eq:appendix_sampling_logit}
\end{equation}
where \(\Delta_{\mathrm{roll},i}\) denotes rollout-induced logit bias. The rollout bias can be locally approximated as
\[\Delta_{\mathrm{roll},i}(k;\hat x_t)
\approx
\sum_{l\in V} A(k,l)n_i^{(W)}(l), \text{ with }A(k,l)>0.\]

\paragraph{Why do semantically similar tokens induce positive contextual contribution?}
\(A(k,l)\) measures contextual association: whether the occurrence of token \(l\) in the recent context provides positive
evidence for predicting token \(k\) at position \(i\).

At the data-distribution level, one can define the contextual contribution as a
log-prior shift:
\begin{equation}
A_i(k,l)
:=
\log
\frac{
p_{\mathrm{data}}(x_0^i=k\mid x_t^r=l,\mathrm{rest})
}{
p_{\mathrm{data}}(x_0^i=k\mid x_t^r=\mathrm{neutral},\mathrm{rest})
}.
\label{eq:contextual_contribution_data}
\end{equation}
Thus, \(A_i(k,l)>0\) means that observing \(l\) in the context increases the conditional prior probability of \(k\). More generally,
tokens in the same semantic cluster tend to have positive pointwise contextual
association under coherent natural language contexts.

At the model level, the same effect can be understood through a local
linearization of the Transformer logits. Suppose the sampling logit is
\[
l_{\theta,i}(k;\hat x_t)
=
u_k^\top h_i(\hat x_t)+b_k,
\]
where \(h_i(\hat x_t)\) is the hidden state at position \(i\) and \(u_k\) is the
output vector for token \(k\). Around a neutral baseline \(\bar x_t\), the
hidden state can be locally approximated as
\[
h_i(\hat x_t)
\approx
h_i(\bar x_t)
+
\sum_{r\in W} B_{ir} e_{\hat x_t^r},
\]
where \(e_{\hat x_t^r}\) is the embedding of the token at position \(r\), and
\(B_{ir}\) summarizes the local attention and feed-forward influence from
position \(r\) to position \(i\). Substituting this into the logit gives
\begin{align}
l_{\theta,i}(k;\hat x_t)-l_{\theta,i}(k;\bar x_t)
&\approx
\sum_{r\in W} u_k^\top B_{ir}e_{\hat x_t^r}.
\end{align}
Hence, the contextual contribution of token \(l\) to the logit of \(k\) can be
approximated by
\begin{equation}
A_{\theta,i}(k,l)
\approx
u_k^\top B_{ir}e_l.
\label{eq:model_level_Akl}
\end{equation}
If \(k\) and \(l\) are semantically related, their embedding and output
directions are often aligned in language models, making \(A\) contribution more
likely to be positive.

\paragraph{Rollout-induced positive feedback.}
The rollout-induced logit bias can
be locally approximated as
\begin{equation}
\Delta_{\mathrm{roll},i}(k;\hat x_t)
\approx
\sum_{r\in W} A(k,\hat x_t^r)
=
\sum_{l\in V} A(k,l)n_i^{(W)}(l).
\label{eq:rollout_feedback_additive_appendix}
\end{equation}
This approximation states that repeated contextual contributions accumulate in
the logit. If \(A(k,l)>0\) for semantically related tokens \(k,l\in C\), then
over-production of cluster \(C\) increases the logits of tokens in \(C\). Thus,
\[
n_i^{(W)}(C)\uparrow
\Rightarrow
\Delta_{\mathrm{roll},i}(k)\uparrow \text{ for } k\in C
\Rightarrow
p_\theta(x_0^i\in C\mid \hat x_t)\uparrow.
\]
This is the rollout-induced positive feedback that drives semantic basin
formation.

\paragraph{Likelihood amplification.}
Let \(C\subseteq V\) be a semantic cluster and suppose the current token
\(j=\hat x_t^i\) belongs to \(C\). For a semantic kernel, we have
\(q_t(j\mid k)\ \text{large for } k\in C\)
and \(q_t(j\mid k)\ \text{small for } k\notin C\).
Define
\(a_{\mathrm{in}}
:=
\inf_{k\in C} q_t(j\mid k)\)
and \(
a_{\mathrm{out}}
:=
\sup_{k\notin C} q_t(j\mid k).
\)
If \(a_{\mathrm{in}}>a_{\mathrm{out}}\), then the posterior odds of cluster
\(C\) are amplified relative to the contextual prior odds:
\begin{equation}
\begin{aligned}
\frac{
q(x_0^i\in C\mid \hat x_t)
}{
q(x_0^i\notin C\mid \hat x_t)
} =
\frac{
\sum_{k\in C}
p_{\mathrm{data}}(k\mid \hat x_t^{-i})q_t(j\mid k)
}{
\sum_{k\notin C}
p_{\mathrm{data}}(k\mid \hat x_t^{-i})q_t(j\mid k)
}
\ge
\frac{a_{\mathrm{in}}}{a_{\mathrm{out}}}
\cdot
\frac{
p_{\mathrm{data}}(x_0^i\in C\mid \hat x_t^{-i})
}{
p_{\mathrm{data}}(x_0^i\notin C\mid \hat x_t^{-i})
}.
\label{eq:cluster_likelihood_amplification}
\end{aligned}
\end{equation}
Thus, the semantic likelihood itself increases the cluster-level log-odds. For a highly local semantic kernel, \(a_{\mathrm{out}}\) can be very small, so the likelihood amplification can be large.

Combining the semantic likelihood amplification in
\Cref{eq:cluster_likelihood_amplification} and the rollout-induced feedback in
\Cref{eq:rollout_feedback_additive_appendix}, semantic diffusion has two significant
sources of cluster amplification, which makes semantic basins more
severe in SemDLM.

\subsection{Global Jumping as Basin Escape}
\label{apdx:global_jumping}

We show why the global jumping
kernel can mitigate semantic basin. Consider a kernel,
\begin{equation}
q_t(j\mid k,c)
=
\alpha_t\delta_{kj}
+
\beta_t\nu_t(j)
+
(1-\alpha_t-\beta_t)s_t^{\mathrm{sem}}(j\mid k,c).
\label{eq:appendix_global_jump_kernel}
\end{equation}
The global component gives \( q_t(j\mid k,c) \ge \beta_t\nu_t(j)\)
As a result, even if \(k\) is outside the semantic neighborhood of \(j\), it still receives non-zero likelihood.

Let \(j\in C\), \(k_{\mathrm{in}}\in C\) and \(k_{\mathrm{out}}\notin C\). With global jumping, 
\(q_t(j\mid k_{\mathrm{out}},c)\ge \beta_t\nu_t(j),\)
and thus the likelihood ratio is bounded by
\begin{equation}
\frac{
q_t(j\mid k_{\mathrm{in}},c)
}{
q_t(j\mid k_{\mathrm{out}},c)
}
\le
\frac{
\alpha_t\delta_{k_{\mathrm{in}}j}
+
\beta_t\nu_t(j)
+
(1-\alpha_t-\beta_t)s_t^{\mathrm{sem}}(j\mid k_{\mathrm{in}},c)
}{
\beta_t\nu_t(j)
}.
\label{eq:global_jump_ratio_bound}
\end{equation}
Therefore, global jumping caps the log-likelihood advantage of the local
semantic cluster, which prevents the positive feedback introduced in~\Cref{sec:desired_properties}.

\subsection{Semantic-Frequency Penalty as Negative Feedback}
\label{apdx:semantic_frequency_penalty}

We rewrite the
semantic frequency penalty term:
\begin{equation}
m_i^{(W)}(k)
=
\sum_{l\in V} S_+(k,l)n_i^{(W)}(l),
\label{eq:appendix_semantic_frequency}
\end{equation}
where \(S_+(k,l)\ge 0\) measures positive semantic association. with the semantic frequency penalty, the new logits are
\begin{equation}
\tilde l_{\theta,i}(k)
=
l_{\theta,i}(k)
-
\lambda_{\mathrm{freq}}\psi(n_i^{(W)}(k))
-
\lambda_{\mathrm{sem}}\psi(m_i^{(W)}(k)),
\qquad
\psi(u)=\log(1+u).
\label{eq:appendix_hybrid_penalty}
\end{equation}
For two tokens \(k\) and \(l\), the corrected log-ratio satisfy
\[
\begin{aligned}
\log
\frac{
\tilde p_\theta(k)
}{
\tilde p_\theta(l)
}
&=
\log
\frac{
p_\theta(k)
}{
p_\theta(l)
}
-
\lambda_{\mathrm{freq}}
\left[
\psi(n_i^{(W)}(k))-\psi(n_i^{(W)}(l))
\right]
\\
&\quad
-
\lambda_{\mathrm{sem}}
\left[
\psi(m_i^{(W)}(k))-\psi(m_i^{(W)}(l))
\right].
\end{aligned}
\]
Therefore, a token receives lower relative odds if either the token itself has
been over-produced, or its semantic neighborhood has been over-produced. We can connect this penalty to the rollout-induced positive feedback. Suppose the semantic feedback satisfies
\begin{equation}
\bar\Delta_{\mathrm{roll}}(C)
-
\bar\Delta_{\mathrm{roll}}(C^c)
\le
\eta_{\mathrm{sem}}
\left[
\psi(m_i^{(W)}(C))-\psi(m_i^{(W)}(C^c))
\right],
\label{eq:semantic_feedback_bound}
\end{equation}
where \(\bar\Delta_{\mathrm{roll}}(C)\) denotes the average rollout-induced
logit shift for tokens in \(C\). After applying the semantic penalty, the net
cluster-level semantic feedback is bounded by
\[
\begin{aligned}
&
\left[
\bar\Delta_{\mathrm{roll}}(C)
-
\bar\Delta_{\mathrm{roll}}(C^c)
\right]
-
\lambda_{\mathrm{sem}}
\left[
\psi(m_i^{(W)}(C))-\psi(m_i^{(W)}(C^c))
\right]
\\
&\le
-(\lambda_{\mathrm{sem}}-\eta_{\mathrm{sem}})
\left[
\psi(m_i^{(W)}(C))-\psi(m_i^{(W)}(C^c))
\right].
\end{aligned}
\]
Thus, when \(\lambda_{\mathrm{sem}}\ge \eta_{\mathrm{sem}}\), the semantic
penalty cancels the positive semantic feedback.

\paragraph{Pure frequency penalty as a special case.}
The exact frequency penalty is recovered by setting \(\lambda_{\mathrm{sem}}=0\)
in \Cref{eq:appendix_hybrid_penalty}, or equivalently by choosing
\(S_+(k,l)=\mathbb I\{k=l\}\). In this case,
\[
m_i^{(W)}(k)=n_i^{(W)}(k),
\]
and the semantic penalty reduces to an exact-token repetition penalty. 
\section{Additional Experimental Details}

\subsection{Practical Designs}
\label{sec:knn_approximation}

\paragraph{Model setup.}
Following~\citet{arriola2025block} and~\citet{sahoo2024simple}, we use a
Transformer backbone with rotary positional embeddings. We adopt the small-scale
architecture of~\citet{arriola2025block}, resulting in a 110M-parameter model.

\paragraph{Practical SemDLM+ kernel.}
Since a full-vocabulary semantic softmax in~\Cref{eq:global_jump_kernel} is expensive
for large vocabularies, we use a top-\(k\) nearest neighborhoods as a practical approximation, i.e.
\begin{equation}
    q(x_t = j \mid x_0 = i)
    =
    \alpha_t \delta_{ij}
    +
    \beta_t \nu_t(j)
    +
    \frac{1-\alpha_t-\beta_t}{k_t}
    \mathbb I\bigl(j\in \mathcal N_{k_t}(i)\bigr),
\end{equation}
where $\mathcal N_{k_t}(i)$ is the top-$k_t$ semantic neighborhood of token $i$. We exclude $i$ from $\mathcal N_{k_t}(i)$ so that $\alpha_t$ exclusively controls self-retention. A simple schedule is
\begin{equation}
    \alpha_t = \alpha_{\min} + (1 - \alpha_{\min}) (1 - t)^\beta,
    \qquad
    \beta_t = \beta_{\max} t^\eta,
    \qquad
    k_t = 1 + (k_{\max} - 1) t^\gamma,
\end{equation}
with $\alpha_{\min}\in(0,1)$ and $\beta,\eta,\gamma>0$. Early diffusion ($t\to 1$) uses larger neighborhoods and a stronger shared transition term for exploration and overlap; late diffusion ($t\to 0$) increases self-retention, shrinks the neighborhood, and can anneal $\beta_t$ downward for precise refinement.

This kernel keeps the main properties of SemDLM+: the semantic neighborhood
keeps the posterior concentrated, while the global transition branch provides shared
support and improves sampling-side repair. In implementation, we store only a
finite top-\(k_{\max}\) neighbor table and absorb the remaining probability mass
into the global transition branch, which keeps memory usage manageable.

\paragraph{Dataset-specific global transition.}
We also test the unigram global transition  For
LM1B, a simple uniform transition  is sufficient. For OpenWebText, we use
a unigram transition where the transition weights is determined by the word frequency in the dataset:
\begin{equation}
    \nu_{\tau}^{\mathrm{owt}}(j)
    =
    \frac{\mathbb I\{j\in \mathcal V_{\mathrm{valid}}^{\mathrm{owt}}\}\,\hat\nu(j)^\tau}
    {\sum_{k\in \mathcal V_{\mathrm{valid}}^{\mathrm{owt}}}\hat\nu(k)^\tau},
    \qquad 0<\tau\le 1,
    \label{eq:owt_tempered_refresh}
\end{equation}
where \(\hat\nu\) is the empirical unigram distribution and
\(\mathcal V_{\mathrm{valid}}^{\mathrm{owt}}\) removes special symbols,
continuation wordpieces, punctuation-only tokens, and extremely rare tokens.
This avoids the unnatural proposals produced by uniform transition on open-domain text.

\paragraph{Changed-aware training objective.}
In practice, many positions remain unchanged under the forward kernel. If they
are weighted equally, the model can overuse a trivial copy shortcut. We therefore
upweight corrupted positions during training. Let \(m_i\in\{0,1\}\) indicate
whether position \(i\) was changed by the forward kernel. We optimize
\begin{equation}
    \mathcal L_{\mathrm{train}}
    =
    \frac{\sum_i w_i \, \ell_i}{\sum_i w_i},
    \qquad
    w_i
    =
    \lambda_{\mathrm{chg}}\mathbb I\{m_i=1\}
    +
    \lambda_{\mathrm{same}}\mathbb I\{m_i=0\},
    \label{eq:practical_changed_loss}
\end{equation}
with \(\lambda_{\mathrm{chg}} \gg \lambda_{\mathrm{same}}\). This preserves the
semantic denoising signal while keeping optimization stable.

\paragraph{Blockwise predictor sampler.}
For generation, we use a semi-autoregressive blockwise sampler. At each step,
the model predicts \(p_\theta(x_0\mid x_t)\), samples a truncated clean proposal,
projects it to the next lower-noise level using the same practical forward
kernel, and updates only low-confidence positions. The replacement probability
is approximated by
\begin{equation}
    \Pr(\text{replace at position }i)
    \approx
    \frac{p_t-p_s}{p_t}
    \Big(1-p_\theta(x_t^i \mid x_t)\Big)^\rho,
    \label{eq:practical_replace_prob}
\end{equation}
optionally with a freeze threshold for already confident tokens.

\subsection{Addtional Figures}
\Cref{fig:repair_curves} draws the impact of sampling steps on generation quality.

\begin{figure*}[t]
  \centering
  \begin{subfigure}[t]{0.45\textwidth}
    \centering
    \includegraphics[width=\linewidth]{figures/lm1b_repair_curve.pdf}
    \caption{LM1B.}
    \label{fig:repair_curve_lm1b}
  \end{subfigure}
  \hfill
  \begin{subfigure}[t]{0.45\textwidth}
    \centering
    \includegraphics[width=\linewidth]{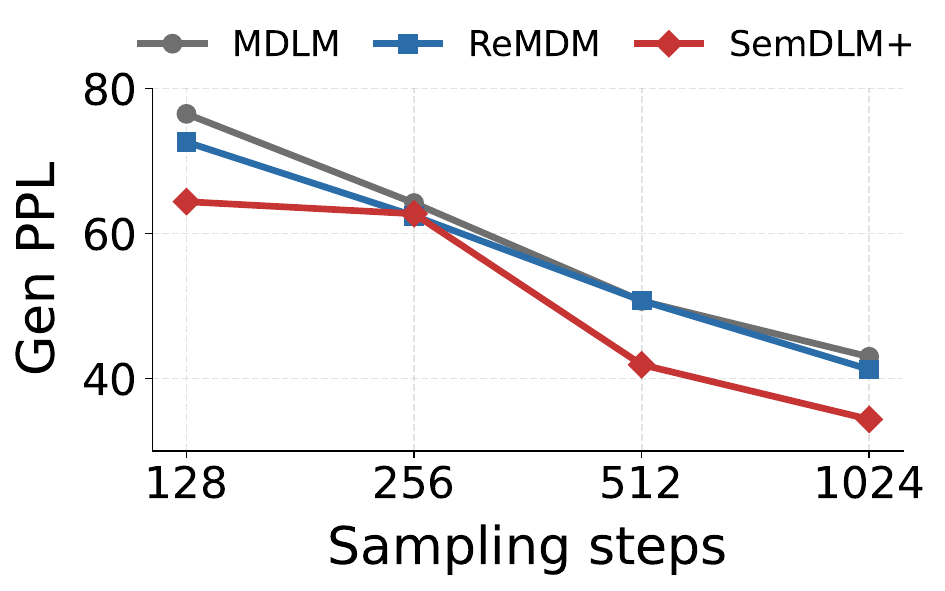}
    \caption{OpenWebText.}
    \label{fig:repair_curve_owt}
  \end{subfigure}
  \caption{Generation PPL as the number of sampling steps increases. SemDLM+ shows stronger improvement with additional refinement steps, indicating better sampling-side repair.}
  \label{fig:repair_curves}
\end{figure*}

\subsection{Computational Cost}

The experiments were run on 8x NVIDIA A100-SXM4-80GB GPUs. For LM1B experiments, the training time is typically 72 hours and for OWT the training time is 144 Hours. The sampling takes approximately 5 mins in the same infrastructure.

\newpage

\end{document}